%% file: bare_adv.tex
\documentclass[10pt,journal,compsoc]{IEEEtran}

\usepackage[utf8]{inputenc} 
\usepackage{hyperref}       
\usepackage{url}            
\usepackage{booktabs}       
\usepackage{amsfonts}       
\usepackage{nicefrac}       
\usepackage{microtype}      
\usepackage[dvipsnames]{xcolor}
\usepackage{subcaption}
\usepackage{graphicx}
\usepackage{natbib}
\usepackage{amsthm}
\usepackage[ruled]{algorithm2e}
\usepackage{float}
\usepackage{makecell}
\newtheorem{theorem}{Theorem}
\newtheorem{corollary}{Corollary}
\newtheorem{lemma}{Lemma}
\newtheorem{definition}{Definition}
\newtheorem{assumption}{Assumption}
\newtheorem{remark}{Remark}

\usepackage{bm}
\usepackage{colortbl}
\usepackage{enumitem}
\usepackage{framed}
\usepackage{tcolorbox}
\usepackage{setspace}
\usepackage{array}
\usepackage{wrapfig}
\usepackage{amssymb}
\usepackage{multirow} 
\usepackage{mathtools} 

\newcolumntype{P}[1]{>{\centering\arraybackslash}p{#1}}

\newcommand{\redformat}[2]{{{}{}}$#1$ ({\color{red}$ \small #2$})}
\newcommand{\greenformat}[2]{{{}{}}$#1$ ({\color{ForestGreen}$ \small #2$})}
\newcommand{\grayformat}[2]{{{}{}}$#1$ ({\color{gray}$ \small #2$})}

\newcommand{\blue}[1]{\begin{color}{blue}#1\end{color}}

\ifCLASSOPTIONcompsoc
  \usepackage[nocompress]{cite}
\else
  \usepackage{cite}
\fi
\ifCLASSINFOpdf
\else
\fi
\usepackage{amsmath}
\begin{document}
%
\title{State Diversity Matters in \\ Offline Behavior Distillation}
%
%
%
%
\author{Shiye~Lei,
Zhihao Cheng,
        and~Dacheng~Tao,~\IEEEmembership{Fellow,~IEEE}

\thanks{S. Lei and Z. Cheng are with the School of Computer Science, Faculty of Engineering, The University of Sydney, Darlington, NSW 2008, Australia. D. Tao is with the Generative AI Lab, College of Computing and Data Science, Nanyang Technological University, 639798, Singapore. (e-mail: shiye.lei@sydney.edu.au, zhihaocheng111@gmail.com, dacheng.tao@gmail.com
).
}
}

%
%

\markboth{ }%
{Lei \MakeLowercase{\textit{et al.}}: Bare Advanced Demo of IEEEtran.cls for IEEE Computer Society Journals}
%



\IEEEtitleabstractindextext{%
\begin{abstract}
Offline Behavior Distillation (OBD), which condenses massive offline RL data into a compact synthetic behavioral dataset, offers a promising approach for efficient policy training and can be applied across various downstream RL tasks. In this paper, we uncover a misalignment between original and distilled datasets, observing that a high-quality original dataset does \textbf{not} necessarily yield a superior synthetic dataset. Through an empirical analysis of policy performance under varying levels of training loss, we show that datasets with greater state diversity outperforms those with higher state quality when training loss is substantial, as is often the case in OBD, whereas the relationship reverses under minimal loss, which contributes to the misalignment. By associating state quality and diversity in reducing pivotal and surrounding error, respectively, our theoretical analysis establishes that surrounding error plays a more crucial role in policy performance when pivotal error is large, thereby highlighting the importance of state diversity in OBD scenario. Furthermore, we propose a novel yet simple algorithm, {\it {\textbf{s}tate \textbf{d}ensity \textbf{w}eighted}} (\textbf{SDW}) OBD, which emphasizes state diversity by weighting the distillation objective using the reciprocal of state density, thereby distilling a more diverse state information into synthetic data. Extensive experiments across multiple D4RL datasets confirm that SDW significantly enhances OBD performance when the original dataset exhibits limited state diversity.
\end{abstract}

}

\maketitle

\IEEEdisplaynontitleabstractindextext

%
\IEEEpeerreviewmaketitle

\input{sec/intro}

\input{sec/related_works}

\input{sec/preliminaries}

\input{sec/analysis}

\input{sec/density_weight_algorithm}

\input{sec/experiments}

\input{sec/discussion}

\appendices
\input{sec/appendix}



%

\bibliographystyle{IEEEtranN}
\bibliography{ref}

%




\end{document}

%% file: sec/intro.tex
\section{Introduction}
\label{sec:intro}

\IEEEPARstart{O}{ffline} reinforcement learning (RL), which learns a policy from a pre-collected dataset and avoids interacting with environments, has received many attention in vision control and language processing due to the safety and convenience \citep{levine2020offline,10610132,snell2023offline}. On the other hand, the static datasets associated with offline RL are typically massive and collected by sub-optimal policies \citep{fu2020d4rl}, impairing the training efficiency. To alleviate this issue, \citet{lei2024offline} propose offline behavior distillation (OBD) that condenses the massive offline RL dataset into a tiny behavioral dataset, and a satisfactory policy can be trained {\it in one second} on this high-informative distilled dataset. Therefore, OBD offers many benefits in (1) training cost reduce and green AI; (2) downstream RL tasks by using distilled data as prior knowledge ({\it e.g.} continual RL \citep{gai2023offline}, multi-task RL \citep{yu2021conservative}, efficient policy pretraining \citep{goecks2019integrating}, offline-to-online RL finetuning \citep{zhao2022adaptive}); and (3) data privacy \citep{qiao2023offline}.



In this paper, we explore how dataset characteristics impact OBD performance, beginning with an investigation into the alignment between original datasets and distilled datasets. 
With two types of offline RL datasets \texttt{M-E} and \texttt{M-R}, where \texttt{M-E} contains more high-quality states, while \texttt{M-R} exhibits greater state diversity, {\it i.e.}, have more diverse states. Although \texttt{M-E} can train better policies than \texttt{M-R}, 
distilled datasets derived from \texttt{M-R} outperform those derived from \texttt{M-E}. This performance misalignment demonstrates that {\it a high-quality original dataset does \textbf{not} necessarily yield a superior distilled dataset}. 
Because policy training on the distilled dataset can be regarded as training on the original dataset with a substantial training loss, we analyze policy performance under varying levels of training loss. Our empirical results remarkably demonstrate that (1) when the training loss is minimal, policies trained on \texttt{M-E} datasets outperform those trained on \texttt{M-R} datasets; (2) however, as training loss increases, policies trained on \texttt{M-R} datasets exhibit better performance than those trained on \texttt{M-E} datasets. These observations suggest that {\it the relative quality between \texttt{M-E} and \texttt{M-R} datasets may vary depending on the magnitude of the training loss}, and {\it state diversity plays a more crucial role when training loss cannot be effectively minimized}, as is often the case in OBD.

Beyond empirical analysis, we further investigate the effect of state diversity from a theoretical perspective. We partition the state space into high-quality pivotal states and diverse surrounding states, depending on whether they are visited by an expert policy, and define the pivotal and surrounding error accordingly. Generally, datasets with higher-quality states are more effective in reducing pivotal error, whereas those with greater state diversity primarily decrease surrounding error. Previous theorem suggests that policy performance depends exclusively on pivotal error, overlooking the contribution of surrounding error (Theorem \ref{thm:expert}). In contrast, we prove that policy performance is influenced by both pivotal and surrounding errors (Theorem \ref{thm:3}), with {\it surrounding error having a greater impact when pivotal error is large}, as often in OBD. This underscores the importance of state diversity in original dataset for reducing surrounding error of policies trained on the corresponding distilled dataset.

Motivated by the significance of state diversity, we propose a novel yet simple algorithm, {\it \textbf{s}tate \textbf{d}ensity \textbf{w}eighted} (\textbf{SDW}) OBD, which emphasizes state diversity of original data during the distillation process. Concretely, SDW weights the OBD loss for each state $s$ by the reciprocal of the state’s density, $\frac{1}{{d}(s)}$. This approach ensures that rare states in the original dataset are assigned larger weights, {\it i.e.}, received greater attention. As a result, more diverse state information can be distilled into the synthetic behavioral dataset, thereby reducing surrounding error of policies trained on it.  
{Extensive experiments on multiple D4RL datasets confirm that SDW significantly enhances OBD performance, particularly when original datasets exhibit limited state diversity.}

Our contributions can be summarized as:
\begin{itemize}[leftmargin=7mm]
    \item We reveal a misalignment between original and distilled datasets in OBD, as datasets with greater state diversity show better resilience to large training loss.
    


    \item We theoretically establish the policy performance guarantee {\it w.r.t.} surrounding error, which underscores the significance of state diversity for OBD.

    \item We propose the novel SDW, which prioritizes state diversity and achieves a superior OBD performance.
\end{itemize}




%% file: sec/related_works.tex
\section{Related Works}

\noindent \textbf{Offline RL \ \ } For many real-world applications such as autonomous driving or healthcare, collecting data through online interaction can be dangerous or prohibitively expensive. Offline reinforcement learning (offline RL) addresses this limitation by learning policies entirely from pre-collected datasets generated by suboptimal behavior policies \citep{lange2012batch, fu2020d4rl}. Compared to conventional online RL, offline RL must explicitly address the distributional shift between the learned policy and the behavior policy. A large family of methods tackles this challenge by constraining policy updates to remain close to the behavior distribution, including behavior-regularized and policy-constrained approaches \citep{fujimoto2021a, tarasov2024revisiting, kumar2019stabilizing, wu2019behavior, fujimoto2019off, kostrikov2022offline}. Another major line of work reduces out-of-distribution errors through conservative or uncertainty-aware value estimation, which penalizes unreliable value predictions for actions not well supported by the dataset \citep{kumar2020conservative, yu2021combo, bai2022pessimistic, rigter2022rambo, nakamoto2023calql, uehara2022pessimistic}. Model-based offline RL further incorporates uncertainty-aware dynamics modeling to ensure reliable extrapolation while maintaining pessimism over model predictions \citep{kidambi2020morel, yu2020mopo}. 
However, while offline RL achieves significant progress, the offline dataset is extremely large 
and contains sensitive information ({\it e.g.} medical history) \citep{qiao2023offline}, which require researchers to consider the training efficiency, data storage, and data privacy in offline RL. 



\vspace{3mm}

\noindent \textbf{Dataset Distillation (DD)} 
aims to distill a large real dataset into a much smaller synthetic one such that models trained on the distilled data generalize comparably to those trained on the original dataset \citep{sachdeva2023data, yu2024dataset, lei2024comprehensive}. By constructing tiny but informative synthetic datasets, DD alleviates issues related to computational inefficiency, storage cost, and data privacy. Two major frameworks dominate the literature. One line of work formulates dataset distillation as a bi-level optimization problem within a meta-learning paradigm, where the synthetic dataset is optimized so that a model trained on it performs well on the real data \citep{wang2018dataset, deng2022remember}. The second line, often referred to as matching-based distillation, aligns synthetic and real data in terms of gradient similarity \citep{zhao2021dataset, zhao2021dsa}, feature statistics \citep{zhao2023distribution, wang2022cafe}, or training trajectories \citep{cazenavette2022dataset, cui2023scaling}. Beyond image data, dataset distillation has been extended to a wide range of modalities, including text \citep{maekawa2023dataset,tao2024textual}, video \citep{wang2024dancing,ding2025condensing}, graph \citep{jin2022graph,jin2022condensing,xu2023kernel,zheng2023structurefree}, time series \citep{liu2024dataset,ding2024condtsf,miao2024less}, speech \citep{ritter2024dataset}, and vision–language data \citep{wu2024visionlanguage}. Building on the efficiency benefits of distilled data, \citet{lei2024offline} propose offline behavior distillation for offline RL, where a compact behavioral dataset is synthesized from large quantities of suboptimal trajectories to enable more efficient policy learning. Unlike supervised learning tasks that rely on standardized benchmark datasets, offline RL datasets vary widely in sources, coverage, and behavior characteristics. The impact of such dataset characteristics on distillation remains insufficiently understood, and our work takes an initial step toward addressing this gap.



%% file: sec/preliminaries.tex
\section{Preliminaries}

\noindent \textbf{Reinforcement Learning} is typically formulated as an episode Markov decision process (MDP) $\langle\mathcal{S}, \mathcal{A}, \mathcal{T}, r, T, d^1\rangle$, where 
state ${s} \in \mathcal{S}$, action ${a} \in \mathcal{A}$, $\mathcal{T}({s}^\prime | {s}, {a})$ is the transition function, $r(s,a)$ is the reward function bounded by $|r ({s}, {a})| \leq R_{\max}$, $T$ is the horizon length, and $d^1(s)$ is the initial state distribution \citep{sutton2018reinforcement}. The goal of RL is to learn a policy $\pi$ that maximizes long-term expected return $J(\pi) = \mathbb{E}_{\pi}\left[\sum_{t=1}^{T} r_t\right]$, where $r_t = r({s}_t, {a}_t)$. For clarity, we primarily use the deterministic policies $\pi(s): \mathcal{S} \rightarrow \mathcal{A}$, while also considering stochastic policies $\pi(a|s):\mathcal{S}\times \mathcal{A}\rightarrow [0,1]$ in theoretical analysis. We define $d_{\pi}^t(s)= \Pr(s_t=s; \pi)$ and $\rho_\pi^t(s, a) = \Pr(s_t=s, a_t=a; \pi)$ as $t$-th step state distribution and state-action distribution, respectively. Then, the state distribution and state-action distribution of $\pi$ are denoted as $d_{\pi}(s) = \frac{1}{T}\sum_{t=1}^T d_{\pi}^t(s)$ and $\rho_\pi(s, a) =  \frac{1}{T}\sum_{t=1}^T  \rho_\pi^t(s, a) $. The action-value function of $\pi$ is $q_\pi(s,a) = \mathbb{E}_\pi\left[\sum_{t=1}^{T} r_{t} \mid s_1=s, a_1=a\right]$, which is the expected return starting from $s$, taking the action $a$. Without accessing environments, offline RL learns policies from a large offline dataset $\mathcal{D}_\texttt{off}=\{({s}_i,{a}_i,{s}'_i,r_i)\}_{i=1}^{N_\texttt{off}}$ with specially designed Bellman backup. Although $\mathcal{D}_\texttt{off}$ is typically collected by the suboptimal behavior policy, offline RL algorithms can recapitulate a near-optimal policy $\pi^*$ and value function $q_{\pi^*}$ from $\mathcal{D}_\texttt{off}$.

\vspace{3mm}

\noindent \textbf{Behavioral Cloning} \citep{pomerleau1991efficient} {can be viewed as a special case of offline RL that focuses exclusively on high-quality data. Given a set of expert demonstrations $\mathcal{D}_\texttt{BC} = \{({s}_i, {a}_i)\}_{i=1}^{N_\texttt{BC}}$,}
the policy network $\pi_\theta$ parameterized by $\theta$ is trained in a supervised manner to mimic the behavior observed in $\mathcal{D}_\texttt{BC}$. Specifically, BC minimizes the following objective: $\min_{\theta} \ell_\texttt{BC}(\theta, \mathcal{D}_\texttt{BC}) \coloneq \mathbb{E}_{({s},{a})\sim \mathcal{D}_\texttt{BC}}\left[ {\left(\pi_\theta \left({a}|{s}\right) -  \hat{\pi}^*(a|s)\right)^2}\right]$, 
where $\hat{\pi}^*(a|s)$ is an empirical estimation of the expert policy, computed as $\hat{\pi}^*(a|s) = \frac{\sum_{i=1}^{N_\texttt{BC}}\mathbb{I}(s_i=s, a_i=a)}{\sum_{i=1}^{N_\texttt{BC}}\mathbb{I}(s_i=s)}$. 

\begin{figure}[t]
\centering
\begin{subfigure}[b]{0.15\textwidth}
    \centering
    \includegraphics[width=\columnwidth]{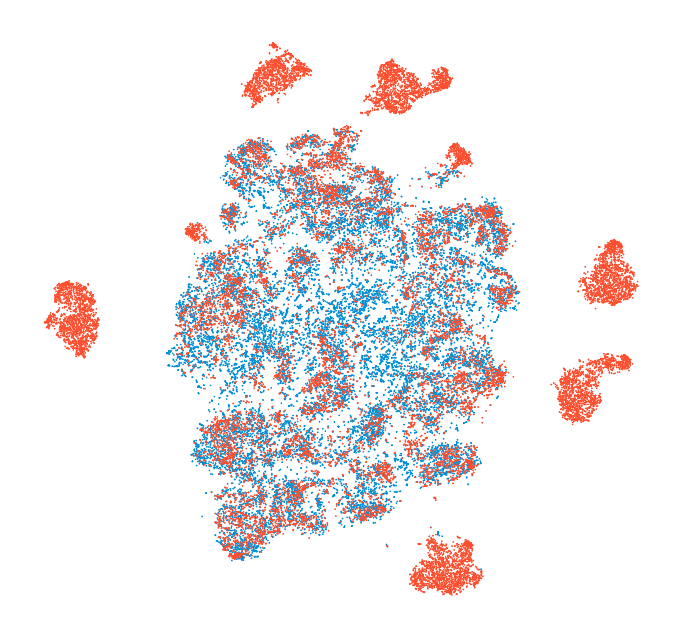}
    \caption{HalfCheetah}
    \label{figure:halfcheetah tsne}
\end{subfigure}%
\hfill
\begin{subfigure}[b]{0.15\textwidth}
    \centering
    \includegraphics[width=\columnwidth]{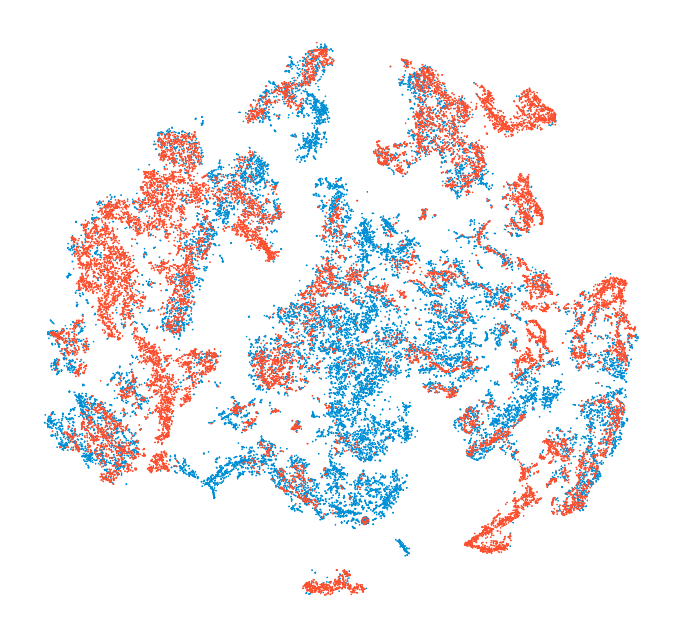}
    \caption{Hopper}
    \label{figure:hopper tsne}
\end{subfigure}%
\hfill
\begin{subfigure}[b]{0.15\textwidth}
    \centering
    \includegraphics[width=\columnwidth]{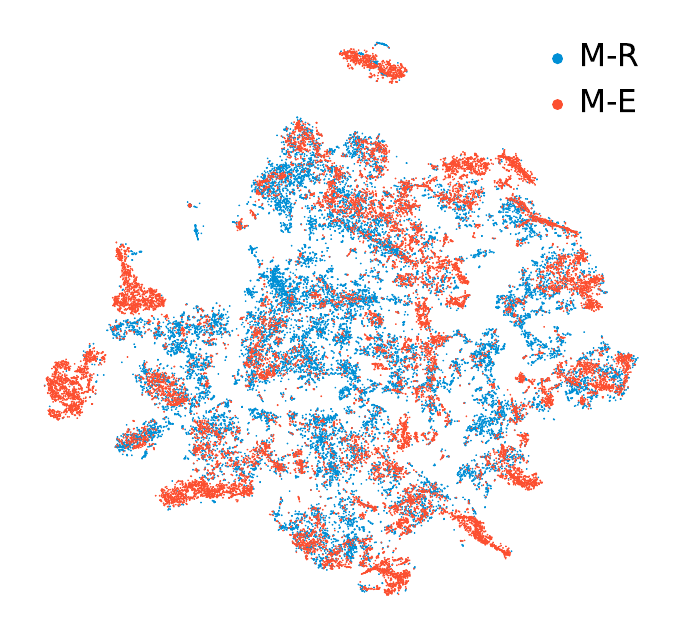}
    \caption{Walker2d}
    \label{figure:walker2d tsne}
\end{subfigure}
\caption{t-SNE visualization of \texttt{M-R} (blue) and \texttt{M-E} (red) states.}
\label{figure:tsne}
\end{figure}

\subsection{Offline Behavior Distillation}
OBD synthesizes a tiny behavioral dataset $\mathcal{D}_\texttt{syn}=\{(s_i, a_i)\}_{i=1}^{N_\texttt{syn}}$ from $\mathcal{D}_\texttt{off}$, where $N_\texttt{syn} \ll N_\texttt{off}$, and policies trained on $\mathcal{D}_\texttt{syn}$ via BC have a high return when interacting with the environment. Because $\mathcal{D}_\texttt{off}$ is collected by suboptimal policies, an effective policy $\pi^\ast$ is learned from $\mathcal{D}_\texttt{off}$ via offline RL algorithms in advance and then relabel actions for $s\in\mathcal{D}_\texttt{off}$, resulting the corrected original dataset $\mathcal{D}_{\texttt{off}|\pi^\ast}=\{s_i, \pi^\ast(s_i)\}_{i=1}^{N_\texttt{off}}$ to guide OBD process, {\it i.e.}, measure the performance of policy trained on $\mathcal{D}_\texttt{syn}$. Let $\mathcal{D}_\texttt{real} = \mathcal{D}_{\texttt{off}|\pi^\ast}$ denote the original/real dataset prior to distillation, then OBD can be formulated as a bi-level optimization problem:
\begin{align}
\label{eq:obd}
     & \mathcal{D}_\texttt{syn}^\ast = \arg\min_{\mathcal{D}} \mathcal{H} \left(\pi_{\theta\left(\mathcal{D}\right)}, \mathcal{D}_\texttt{real}\right) \nonumber \\
& \text{s.t.} \quad   \theta(\mathcal{D}) = \arg\min_\theta \ell_\texttt{BC}(\theta, \mathcal{D}), \nonumber
\end{align}
where $\mathcal{H}(\cdot)$ denotes the OBD objective. 
A naive OBD loss is policy-based cloning (PBC), which measures the difference between $\pi$ and $\pi^\ast$: 
\begin{equation}
    \mathcal{H}_{\texttt{PBC}}=\mathbb{E}_{(s,a)\sim\mathcal{D}_\texttt{real}}\left[\left(\pi(s) - a\right)^2 \right] \nonumber
\end{equation}
\citet{lei2024offline} improve this by proposing action-value weighted PBC (Av-PBC): 
\begin{equation}
    \mathcal{H}_{\texttt{Av-PBC}}=\mathbb{E}_{(s,a)\sim\mathcal{D}_\texttt{real}}\left[q_{\pi^\ast}(s,a)\left(\pi(s) - a\right)^2 \right] \nonumber
\end{equation}
The distillation objective $\mathcal{H}$ can be considered as the (weighted) loss of policy $\pi$ on the original dataset $\mathcal{D}_\texttt{real}$. Then each $s_i$ and $a_i$ in $\mathcal{D}_\texttt{syn}$ are continually updated by decreasing $\mathcal{H}$ via b ackpropagation through time (BPTT) \citep{werbos1990backpropagation}. 

%% file: sec/analysis.tex
\section{Dataset Quality Alignment?}
\label{sec:alignment}

In the case of real datasets distilled by conventional DD in supervised learning domain, data are typically i.i.d. drawn and labeled without noise. Consequently, researchers focus on distilling the standard dataset for each task, hardly exploring how dataset characteristics influence distillation performance. However, in offline RL, there is no standard dataset for each task (environment); instead, datasets are collected using different strategies involving various policies or humans interacting with the environment, leading to variations in dataset quality. This raises a critical question: is there alignment between the quality\footnote{Unless otherwise specified, the quality of dataset refers to the performance of policies trained on it.} of original and distilled data in OBD? Specifically, \textit{can a higher-quality synthetic dataset, $\mathcal{D}_\texttt{syn}$, be distilled from a superior original dataset $\mathcal{D}_\texttt{real}$?} To address this, we begin by presenting the nature of offline RL datasets to provide a clearer context.



\noindent \textbf{Datasets \ \ } 
We utilize offline RL data provided by D4RL \citep{fu2020d4rl}, a widely recognized benchmark for offline RL. Three environments of \texttt{Halfcheetah}, \texttt{Hooper}, \texttt{Walker2D} are utilized for data collection. For each environment, three offline datasets with different qualities are offered by D4RL, {\it i.e.}, \texttt{medium-replay} (\texttt{M-R}), \texttt{medium} (\texttt{M}), and \texttt{medium-expert} (\texttt{M-E}) datasets. Concretely, \texttt{M-R} datasets are generated by a range of policies, from random to medium-level; \texttt{M} datasets are produced by medium-level policies; and \texttt{M-E} datasets are collected by policies ranging from medium to expert level. 
Suboptimal policies explore a border range of states, while expert policies exhibit more consistent behavior. As a result, {\it the \texttt{M-R} datasets offer better state diversity, whereas the \texttt{M-E} datasets are characterized by higher state quality}, as they allow for higher cumulative rewards \citep{monier2020offline,schweighofer2021understanding}. As illustrated in Figure \ref{figure:tsne}, we visualize state distribution of \texttt{M-R} and \texttt{M-E} data via t-SNE \citep{van2008visualizing}. The \texttt{M-R} states evenly cover the central area, while the \texttt{M-E} states from concentrated clusters. 


\begin{table}[t]
    \centering
    \caption{Performance of policies trained on $\mathcal{D}_\texttt{real}$ and $\mathcal{D}_\texttt{syn}$. The better dataset of \texttt{M-R} 
    and \texttt{M-E} 
    is marked with \textbf{bold} scores for each environment.}
    \resizebox{\linewidth}{!}{
    \begin{tabular}{c  cc  cc  cc }
    \toprule
     \multirow{2}{*}{Datasets}  & \multicolumn{2}{c}{Halfcheetah} & \multicolumn{2}{c}{Hopper} & \multicolumn{2}{c}{Walker2D}  \\
     & \texttt{M-R} & \texttt{M-E} & \texttt{M-R} & \texttt{M-E} & \texttt{M-R} & \texttt{M-E} \\
     \specialrule{0.1pt}{0.2\jot}{0.2pc}
      $\mathcal{D}_\texttt{real}$ & $45.1$ & $\bm{53.1}$ & $93.0$ & $\bm{99.0}$ & $84.6$ & $\bm{107.2}$ \\ 
        $\mathcal{D}_\texttt{syn}$ & $\bm{35.9}$ & ${22.0}$ & $\bm{40.9}$ & $38.7$ & $\bm{55.0}$ & ${42.1}$ \\ 
    \bottomrule
    \end{tabular}
    }
    \label{tab:misalign}
\end{table}

\begin{figure*}[t]
\centering
\begin{subfigure}[b]{0.3\textwidth}
    \centering
    \includegraphics[width=\columnwidth]{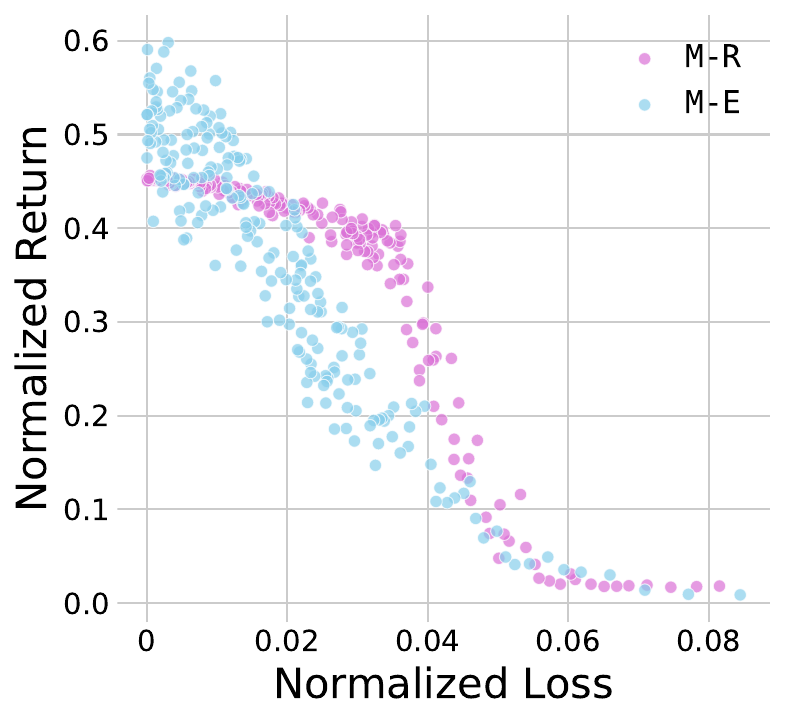}
    \caption{Halfcheetah}
    \label{figure:halfcheetah r}
\end{subfigure}%
\hfill
\begin{subfigure}[b]{0.3\textwidth}
    \centering
    \includegraphics[width=\columnwidth]{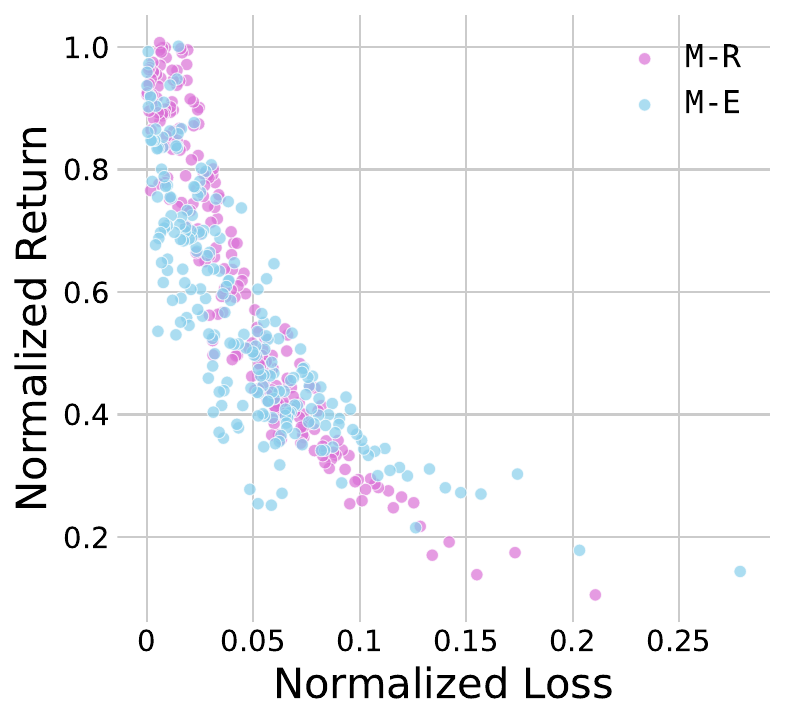}
    \caption{Hopper}
    \label{figure:hopper r}
\end{subfigure}%
\hfill
\begin{subfigure}[b]{0.3\textwidth}
    \centering
    \includegraphics[width=\columnwidth]{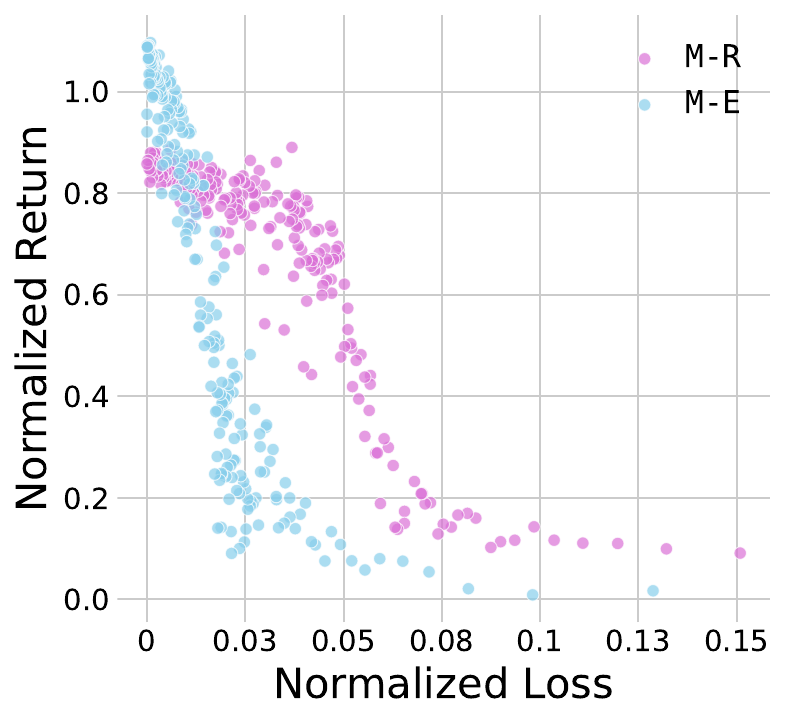}
    \caption{Walker2D}
    \label{figure:walker2d r}
\end{subfigure}
\caption{Plots of policy performance (normalized return) as functions of the loss {\it w.r.t.} real data. Each point is calculated and then averaged over five trials.}
\label{figure:error_vs_return}
\end{figure*}

\vspace{1mm}
\noindent \textbf{Experiments \ \ } 
We train policies on the real \texttt{M-R} and \texttt{M-E} datasets $\mathcal{D}_\texttt{real}$ and their corresponding distilled datasets $\mathcal{D}_\texttt{syn}$\footnote{The distilled behavioral data are synthesized using Av-PBC.
}, respectively. The policy performance results, shown in Table \ref{tab:misalign}, reveal several observations: (1) in the \texttt{Hopper} environment, policies trained on (distilled) \texttt{M-R} and (distilled) \texttt{M-E} datasets exhibit comparable performance; and (2) in the \texttt{Halfcheetah} and \texttt{Walker2D} environments, polices trained on \texttt{M-E} data outperform those trained on \texttt{M-R} data. However, distilled datasets derived from \texttt{M-R} datasets outperform those derived from \texttt{M-E} datasets. These findings suggest that \textit{a high-quality original dataset does \textbf{not} necessarily yield a superior synthetic dataset}. We next investigate this misalignment between the quality of original and distilled datasets in greater detail.

{

\subsection{Empirical Analysis}

When training policies directly on $\mathcal{D}_\texttt{real}$ with the supervised BC, the objective loss $\mathcal{H}$ is readily decreased to near zero, leading to policies that perfectly fit $\mathcal{D}_\texttt{real}$. In contrast, $\mathcal{D}_\texttt{syn}$ is obtained by minimizing $\mathcal{H}$ in a bi-level optimization process, where $\mathcal{H}$ is difficult to minimize to a significant degree. To this end, {\it policies trained on $\mathcal{D}_\texttt{syn}$ can be viewed as training on $\mathcal{D}_\texttt{real}$ with a larger loss $\mathcal{H}$}.

Building on the distinction between training policies on real versus synthetic data, we investigate policy performance under varying training losses $\mathcal{H}$. Specifically, we train policies directly on the real dataset $\mathcal{D}_\texttt{real}$ and monitor both the training loss and the corresponding policy return throughout the training process. The results are presented in Figure \ref{figure:error_vs_return}. In Figure \ref{figure:hopper r} for \texttt{Hopper}, the trend of policy performance relative to the loss on \texttt{M-E} and \texttt{M-R} data is quite similar, which explains their comparable performance in Table \ref{tab:misalign}. From Figures \ref{figure:halfcheetah r} and \ref{figure:walker2d r} on \texttt{Halfcheetah} and \texttt{Walker2D}, respectively, we observe two key trends: (1) under low training loss, policies trained on \texttt{M-E} data outperform those trained on \texttt{M-R} data; (2) as training loss increases, policies trained on \texttt{M-R} data achieve better performance. These observations suggest that {\it while state quality (\texttt{M-E}) is essential for policy performance under marginal loss, state diversity (\texttt{M-R}) plays a critical role in mitigating performance degradation as the loss increases.}


In most learning tasks, objectives can be directly minimized in a single-level manner, and dataset quality is evaluated under the {\it interpolation phase}, where training loss approaches zero. However, the above empirical study reveals that dataset quality in {\it underfitting phase} is not always correlated with that in {\it interpolation phase} in offline RL, and {\it state diversity becomes more crucial to dataset quality in the underfitting phase}. In offline behavior distillation, which employs bi-level optimization, policies trained on synthetic data inevitably incur significant loss {\it w.r.t.} original data, situating the assessment of dataset quality in the {underfitting phase}. The discrepancy between dataset quality in interpolation and underfitting phases helps explain why \texttt{M-R} datasets, despite inferior to \texttt{M-E} datasets in interpolation phase, can yield superior synthetic datasets in OBD.








}

\section{Theoretical Analysis}
In this section, we theoretically analyze why \texttt{M-E} and \texttt{M-R} datasets show inverse relative qualities under interpolation and underfitting phases. Considering that \texttt{M-E} datasets have better state quality, while \texttt{M-R} datasets exhibit greater state diversity as illustrated in Section \ref{sec:alignment}, we partition the state space into high-quality pivotal states and diverse surrounding states, depending on whether they are visited by an expert policy, and define the pivotal and surrounding error accordingly. Generally, datasets with higher-quality states are more effective in reducing the pivotal error, whereas datasets with better state diversity primarily decrease surrounding error. Previous theorem suggests that policy performance depends solely on pivotal error, overlooking the contribution of surrounding error. In contrast, we prove that policy performance is influenced by both pivotal and surrounding errors, with surrounding error exerting a greater impact in the underfitting phase, as often in OBD. This underscores the importance of state diversity in original data for reducing surrounding error of policies trained on the corresponding distilled dataset.


\begin{definition}[Pivotal and surrounding state]
Given the expert policy $\pi^\ast$, the pivotal state set $\mathcal{S}_e(\pi^\ast)=\{s|d_{\pi^\ast}(s)>0\}$, and the corresponding surrounding state set $\mathcal{S}_\mu(\pi^\ast)=\{s|d_{\pi^\ast}(s)=0\}$.
\end{definition}
The state space $\mathcal{S}$ is partitioned into pivotal states $\mathcal{S}_e$ and surrounding states $\mathcal{S}_\mu$,  according to whether they are visited by $\pi^\ast$, thus pivotal states typically exhibit higher quality compared to surrounding states. Given a policy $\hat{\pi}$
, we define the errors of $\hat{\pi}$ on $\mathcal{S}_e$ and $\mathcal{S}_\mu$ as follows.
\begin{definition}
The pivotal error of $\hat{\pi}$ is defined as $\mathcal{E}_{e \vert \pi^\ast}(\hat{\pi}) = \mathbb{E}_{s\sim d_{\pi^\ast}}\left[\sum_{a\in\mathcal{A}} \left\vert \hat{\pi}(a|s) -  \pi^\ast(a|s) \right\vert \right]$.
\end{definition}
Pivotal error quantifies the action difference between $\hat{\pi}$ and $\pi^\ast$ on pivotal states. Since $\pi^\ast$ only visits pivotal states, the action difference cannot be used to measure the error of $\hat{\pi}$ on surrounding states. When $\hat{\pi}$ visits a surrounding state at step $t$, we consider the action on this surrounding state is correct if the action helps transition to a pivotal state at the next step. Then surrounding error can be correspondingly defined as follows.
\begin{definition}
The surrounding error of $\hat{\pi}$, denoted as $\mathcal{E}_{\mu \vert \pi^\ast}(\hat{\pi})$, is the average probability of transitioning to a surrounding state at the next step, given that $\hat{\pi}$ is currently in a surrounding state.
\end{definition}
The previous theorems exclusively shows the significance of pivotal error in policy performance, as outlined below. 
\begin{theorem}[Theorem 2.1 in \citep{ross2010efficient}]
\label{thm:expert}
For two policies $\pi^\ast$ and $\hat{\pi}$, if $\mathcal{E}_{e \vert \pi^\ast}(\hat{\pi}) \leq \epsilon$, we have 
$ |J(\pi^\ast) - J(\hat{\pi})|  \leq  \epsilon T^2 R_{\max}$.
\end{theorem}
We provide a novel proof for this theorem. In contrast to the original, our proof (1) is entirely based on formal operations rather than abstract notions and narratives, and (2) applies to stochastic $\pi^\ast$, whereas the original proof assumes a deterministic expert policy. A proof sketch is presented below, with the full proof detailed in Appendix \ref{app:proof thm1}.

\vspace{1mm}
\noindent \textbf{Proof Sketch\ \ }The left hand side can be expanded into
\begin{align}
     \vert & J(\pi^\ast)  - J(\hat{\pi}) \vert \nonumber \\
     &     
      \leq \sum_{t=1}^T \left\vert \mathbb{E}_{s\sim d_{\pi^\ast}^t, a\sim \pi^\ast}[r(s,a)] - \mathbb{E}_{s\sim d_{\hat{\pi}}^t, a\sim \hat{\pi}}[r(s,a)] \right\vert 
\end{align}
Then for a specific step $t$,
\begin{align}
     &\left\vert \mathbb{E}_{s\sim d_{\pi^\ast}^t, a\sim \pi^\ast}[r(s,a)] - \mathbb{E}_{s\sim d_{\hat{\pi}}^t, a\sim \hat{\pi}}[r(s,a)] \right\vert \nonumber \\
     &\quad \quad \leq   \underbrace{\left\vert \mathbb{E}_{s\sim d_{\pi^\ast}^t, \blue{a\sim \pi^\ast}}[r(s,a)] {- \mathbb{E}_{s\sim d_{\pi^\ast}^t, \blue{a\sim \hat{\pi}}}[r(s,a)]}\right\vert}_{\bm{A}(t)}  \nonumber \\
    & \quad \qquad + \underbrace{\left\vert { \mathbb{E}_{\blue{s\sim d_{\pi^\ast}^t}, a\sim \hat{\pi}}[r(s,a)]} - \mathbb{E}_{\blue{s\sim d_{\hat{\pi}}^t}, a\sim \hat{\pi}}[r(s,a)]\right\vert}_{\bm{B}(t)} \nonumber.
\end{align}
As shown in the above equations, $|J(\pi^\ast) - J(\hat{\pi})|$ can be divided into $\bm{A}(t)$: the action difference between $\pi^\ast$ and $\hat{\pi}$ and $\bm{B}(t)$: the state distribution difference between $d_{\pi^\ast}^t$ and $d_{\hat{\pi}}^t$ at each step. we use $\epsilon_t=\mathbb{E}_{s\sim d_{\pi^\ast}^t}\left[\sum_{a\in\mathcal{A}} \left\vert \hat{\pi}(a|s) -  \pi^\ast(a|s) \right\vert \right]$ to denote the pivotal error at step $t$, and $\epsilon = \frac{1}{T}\sum_{t=1}^T \epsilon_t$. $\bm{A}(t)$ can be easily bounded by $\bm{A}(t) \leq \epsilon_t R_{\max}$. For $\bm{B}(t)$, we derive the iterative formula $\bm{B}(t) \leq \bm{B}(t-1) + \epsilon_{t-1} R_{\max}$, which yields $\bm{B}(t)\leq \sum_{i=1}^{t-1} \epsilon_i R_{\max}$. Therefore $\bm{A}(t)+\bm{B}(t)\leq \sum_{i=1}^t \epsilon_i R_{\max}$. Finally, $\left\vert J(\pi^\ast) - J(\hat{\pi}) \right\vert \leq \sum_{t=1}^T \sum_{i=1}^t \epsilon_i R_{\max} \leq T\sum_{t=1}^T \epsilon_t R_{\max} = \epsilon T^2 R_{\max}$.  \qed 


Theorem \ref{thm:expert} ensures that the policy $\hat{\pi}$, trained on $\mathcal{D}_\texttt{syn}$, performs well if it makes fewer errors on pivotal states $s\in \mathcal{S}_e$, {\it i.e.}, small pivotal error $\mathcal{E}_{e \vert \pi^\ast}(\hat{\pi})$. However, this does not account for the benefit of diverse surrounding states in \texttt{M-R} datasets, which helps decrease surrounding error, on the policy performance $J(\hat{\pi})$. Next, we theoretically investigate the contributions of state diversity by demonstrating how surrounding error influences $J(\hat{\pi})$.




\begin{assumption}
\label{asp:1}
    For all $t\in[1,T]$, the 
    inequality holds:
\begin{equation}
    \sum_{s\in\mathcal{S}_e(\pi^\ast)} \left(d_{\pi^\ast}^t(s) - d_{\hat{\pi}}^t(s)\right) \left(R_{\max} - \mathbb{E}_{a\sim \hat{\pi}} \left[r(s,a)\right] \right) \geq 0 \nonumber
\end{equation} 
\end{assumption}
\begin{remark}
Assumption \ref{asp:1} suggests that when the term $d_{\pi^\ast}^t(s) - d_{\hat{\pi}}^t(s)$ is negative, the corresponding reward $\mathbb{E}_{a\sim \hat{\pi}} \left[r(s,a)\right]$ should be close to $R_{\max}$, ensuring that the product is large. This can be derived by a more strict assumption: $d_{\hat{\pi}}^t(s)  \leq d_{\pi^\ast}^t (s)$ holds for all $s\in \mathcal{S}_e^t$ and $t\in [1,T]$, which implies that the learned policy $\hat{\pi}$ visits expert states no more frequently than the expert policy $\pi^\ast$.
\end{remark}

Then with Assumption \ref{asp:1}, we prove the following theorems {\it w.r.t.} the surrounding error.

\begin{theorem}[Main Theorem]
\label{thm:3}
    For policies $\pi^\ast$ and $\hat{\pi}$, if $\mathcal{E}_{e \vert \pi^\ast}(\hat{\pi}) \leq \epsilon$ and $\mathcal{E}_{\mu \vert \pi^\ast}(\hat{\pi}) \leq \epsilon_\mu$, and Assumption \ref{asp:1} holds, then we have $|J(\pi^\ast) - J(\hat{\pi})|   \leq  (\epsilon_\mu T + 3) \epsilon T R_{\max}$.
\end{theorem}
\begin{proof}
    We follow the proof of Theorem \ref{thm:expert} with $| J(\pi^\ast) - J(\hat{\pi})| \leq \sum_{t=1}^T A(t) + B(t)$ and $A(t) \leq \epsilon_t R_{\max}  $. We then proceed from $B(t)\leq R_{\max} \sum_{s\in\mathcal{S}} | d_{\pi^\ast}^t(s) - d_{\hat{\pi}}^t(s)|$. With Assumption \ref{asp:1}, we have
\begin{align}
\label{eq:ct}
     \sum_{s\in\mathcal{S}} & | d_{\pi^\ast}^t(s) - d_{\hat{\pi}}^t(s) | \nonumber \\
     & = \sum_{s\in\mathcal{S}_e} | d_{\pi^\ast}^t(s) - d_{\hat{\pi}}^t(s)|  + \sum_{s\in\mathcal{S}_\mu} |d_{\pi^\ast}^t(s) - d_{\hat{\pi}}^t(s)| \nonumber \\
    & = \sum_{s\in\mathcal{S}_e} d_{\pi^\ast}^t(s) - \sum_{s\in\mathcal{S}_e}  d_{\hat{\pi}}^t(s) + \sum_{s\in\mathcal{S}_\mu}  d_{\hat{\pi}}^t(s) \nonumber \\
    & = 1 -  \sum_{s\in\mathcal{S}_e}  d_{\hat{\pi}}^t(s) + \sum_{s\in\mathcal{S}_\mu}  d_{\hat{\pi}}^t(s) \nonumber \\
    & = 2 \sum_{s\in\mathcal{S}_\mu}  d_{\hat{\pi}}^t(s) 
\end{align}
The second equation follows from Assumption \ref{asp:1} of $d_{\pi^\ast}^t (s) \geq d_{\hat{\pi}}^t(s), \forall s\in \mathcal{S}_e^t$, and from the fact that $d_{\pi^\ast}^t (s)=0, \forall s\in  \mathcal{S}_\mu^t$. The third equation is derived from $\sum_{s\in\mathcal{S}_e} d_{\pi^\ast}^t(s) = 1$.

The term $ \sum_{s\in\mathcal{S}_\mu}  d_{\hat{\pi}}^t(s) $ represents the probability of $\hat{\pi}$ visiting surrounding states at step $t$, denoted as $\mathbf{P}_{\hat{\pi}}(s_t \in \mathcal{S}_\mu)$. Here, we slightly abuse the notation by using $s_t$ to represent the state $s$ in step $t$. Since the previous state $s_{t-1}$ can be either a pivotal state or a surrounding state, the event $s_t \in \mathcal{S}_\mu$ can be further decomposed into two cases: (1) $s_{t-1} \in \mathcal{S}_\mu\cap s_t \in S_\mu$, and (2) $s_{t-1} \in \mathcal{S}_e \cap s_t \in S_\mu$. 
\begin{align}
\label{eq:core}
     & \sum_{s\in\mathcal{S}_\mu}  d_{\hat{\pi}}^t(s) \nonumber \\
     & = \blue{\mathbf{P}_{\hat{\pi}}(s_t \in \mathcal{S}_\mu)} \nonumber \\
     & = \mathbf{P}_{\hat{\pi}}(s_{t-1} \in \mathcal{S}_\mu, s_{t} \in \mathcal{S}_\mu) 
     + \mathbf{P}_{\hat{\pi}}(s_{t-1} \in \mathcal{S}_e, s_{t} \in \mathcal{S}_\mu)  \nonumber \\
     & = \blue{\mathbf{P}_{\hat{\pi}}(s_{t-1} \in \mathcal{S}_\mu)} \mathbf{P}_{\hat{\pi}}(s_{t} \in \mathcal{S}_\mu \vert s_{t-1} \in \mathcal{S}_\mu) \nonumber \\
     & \quad + \mathbf{P}_{\hat{\pi}}(s_{t-1} \in \mathcal{S}_e) \mathbf{P}_{\hat{\pi}}(s_{t} \in \mathcal{S}_\mu \vert s_{t-1} \in \mathcal{S}_e)
\end{align}
Therefore, with $C(t)=\sum_{s\in\mathcal{S}} \left\vert d_{\pi^\ast}^t(s) - d_{\hat{\pi}}^t(s)\right\vert = 2 \mathbf{P}_{\hat{\pi}}(s_t \in \mathcal{S}_\mu)$, we have
\begin{align}
\label{eq:iter2}
    \blue{C(t)} \nonumber  & =  \blue{C(t-1)} \mathbf{P}_{\hat{\pi}}(s_{t} \in \mathcal{S}_\mu \vert s_{t-1} \in \mathcal{S}_\mu)  \nonumber \\
    & \quad + 2 \mathbf{P}_{\hat{\pi}}(s_{t-1} \in \mathcal{S}_e) \mathbf{P}_{\hat{\pi}}(s_{t} \in \mathcal{S}_\mu \vert s_{t-1} \in \mathcal{S}_e)
\end{align}
The term $\mathbf{P}_{\hat{\pi}}(s_{t} \in \mathcal{S}_\mu \vert s_{t-1} \in \mathcal{S}_\mu) $ represents the probability of transitioning from a surrounding state to a surrounding state at step $t-1$, with the bound $\mathbf{P}_{\hat{\pi}}(s_{t} \in \mathcal{S}_\mu \vert s_{t-1} \in \mathcal{S}_\mu)\leq \epsilon_\mu^{t-1}$. Similarity, $\mathbf{P}_{\hat{\pi}}(s_{t} \in \mathcal{S}_\mu \vert s_{t-1} \in \mathcal{S}_e)$ denotes the probability of transitioning from an expert state to a surrounding state at step $t-1$, and only the policy $\hat{\pi}$ behavior different. Thus, $\mathbf{P}_{\hat{\pi}}(s_{t} \in \mathcal{S}_\mu \vert s_{t-1} \in \mathcal{S}_e) \leq \epsilon_{t-1}$. By plugging the above equation into Eq. \ref{eq:iter2} with $\mathbf{P}_{\hat{\pi}}(s_{t} \in \mathcal{S}_\mu \vert s_{t-1} \in \mathcal{S}_\mu)\leq \epsilon_\mu^{t-1}$ and $\mathbf{P}_{\hat{\pi}}(s_{t} \in \mathcal{S}_\mu \vert s_{t-1} \in \mathcal{S}_e) \leq \epsilon_{t-1}$, and $C(t-1) \leq \sum_{i=1}^{t-2} \epsilon_i$ (Eq. \ref{eq:iterative}), for $T\geq 3$, we have
\begin{align}
    C(t) \leq \epsilon_\mu^{t-1} \sum_{i=1}^{t-2} \epsilon_i +   2 \epsilon_{t-1} \nonumber
\end{align}
Therefore,
\begin{align}
    \vert & J(\pi^\ast) -  J(\hat{\pi})\vert \nonumber \\
    & = \sum_{t=1}^T A(t) + B(t) \nonumber\\
    & \leq R_{\max} (2\epsilon_1 + \epsilon_2) + R_{\max} \sum_{t=3}^T (\epsilon_t + \epsilon_\mu^{t-1} \sum_{i=1}^{t-2} \epsilon_i + 2 \epsilon_{t-1}) \nonumber \\
    & \leq 3\epsilon T R_{\max} + R_{\max} \sum_{t=3}^T (\epsilon_\mu^{t-1} \sum_{i=1}^{t-2} \epsilon_i) \nonumber\\
    & \leq 3\epsilon T R_{\max} + \epsilon_\mu \epsilon T^2 R_{\max} \nonumber
\end{align}

The proof is completed.
\end{proof}
As $\epsilon$ approaches zero, the bound in Theorem \ref{thm:3} similarly converges to zero, making it non-vacuous {\it w.r.t.} the pivotal error $\epsilon$. Furthermore, Theorem \ref{thm:3} provides a tighter bound than Theorem \ref{thm:expert} when $\epsilon_\mu \leq \frac{T-3}{T}$, a condition commonly met when $T$ is large.

\begin{figure}[t]
\centering
\begin{subfigure}[b]{0.23\textwidth}
    \centering
    \includegraphics[width=1.06\columnwidth]{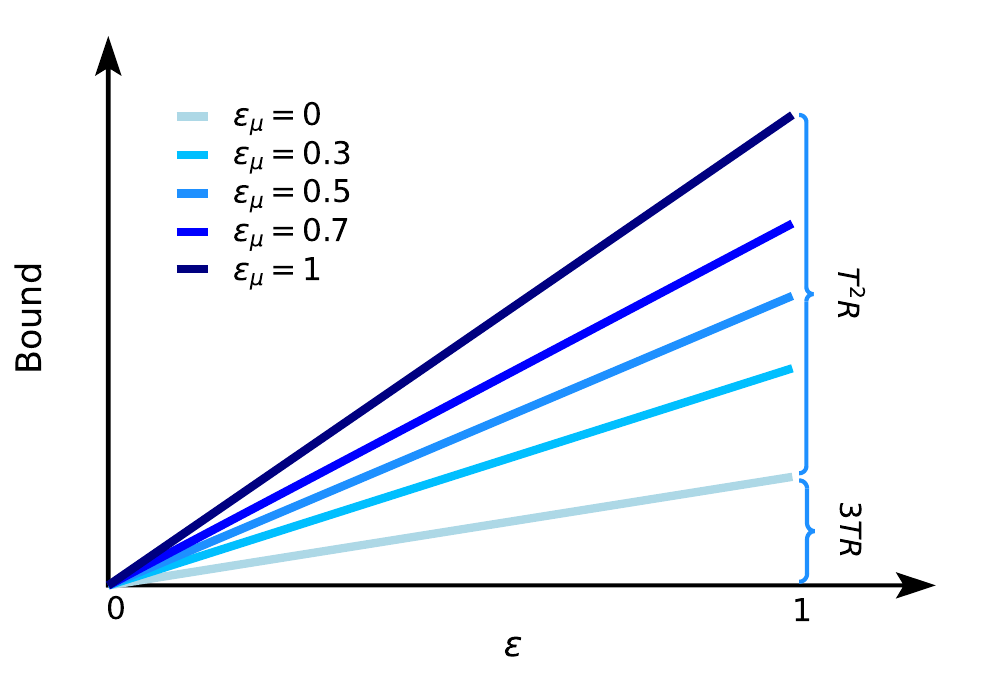}
    \caption{}
    \label{figure:thm eps}
\end{subfigure}
\hfill
\begin{subfigure}[b]{0.24\textwidth}
    \centering
    \includegraphics[width=1.06\columnwidth]{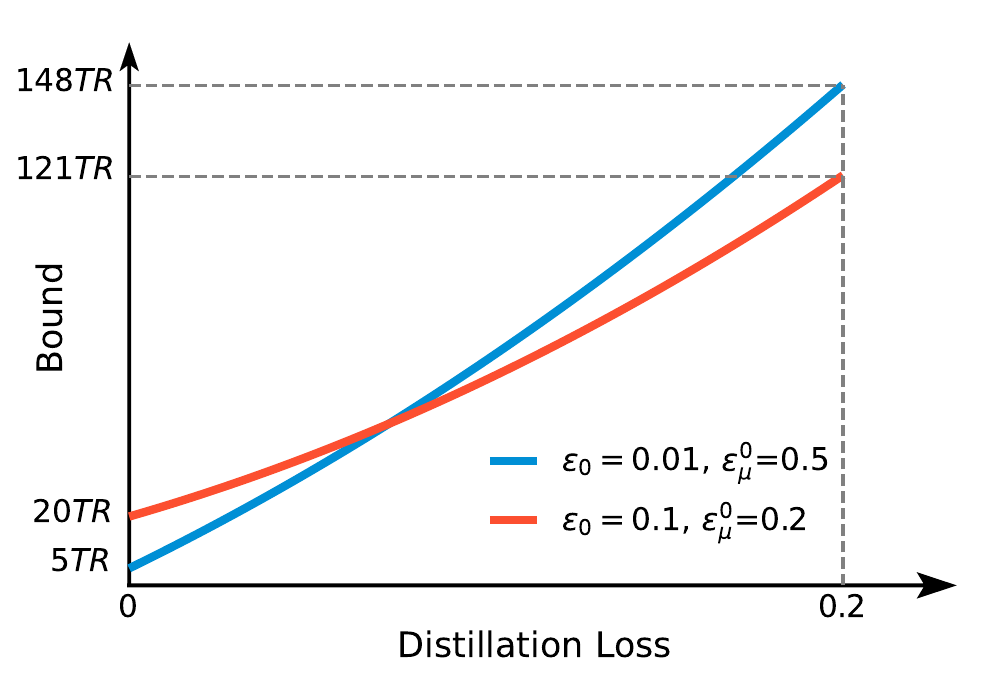}
    \caption{}
    \label{figure:thm error}
\end{subfigure}
\caption{Policy gap bound as a function of (a) pivotal error $\epsilon$ under different surrounding error $\epsilon_\mu$; (b) distillation loss. $R$ denotes $R_{\max}$ for simplicity.}
\label{figure:theorem}
\end{figure}

\vspace{2mm}
\noindent \textbf{Effect of $\epsilon_\mu$ \ \ }  Theorem \ref{thm:3} establishes that the performance gap between $\hat{\pi}$ and $\pi^\ast$ is bounded by $(\epsilon_\mu T +3) \epsilon T R_{\max}$, where $\hat{\pi}$ has a pivotal error of $\epsilon$ and surrounding error $\epsilon_\mu$. To illustrate the impact of surrounding error $\epsilon_\mu$, we plot the relationship between the bound and pivotal error $\epsilon$ under varying levels of $\epsilon_\mu$, as shown in Figure \ref{figure:thm eps}. When $\epsilon$ approaches zero, the bound also approaches zero, and the influence of $\epsilon_\mu$ is marginal. However, as the pivotal error increases, the effect of $\epsilon_\mu$ on the policy gap becomes significant, with large $\epsilon_\mu$ substantially increasing the bound ({\it e.g.} $3TR \rightarrow (T+3)TR$ when $\epsilon=1$). 
In OBD, the pivotal error cannot be minimized to such a small value due to the complex bi-level optimization. Consequently, surrounding error play a more crucial role in determining the performance of policies trained on distilled data.

\vspace{2mm}

\noindent \textbf{Effect of Real Dataset \ \ } 
Given the original dataset $\mathcal{D}_\texttt{real}$, if a policy can perfectly interpolate $\mathcal{D}_\texttt{real}$, it will exhibit the pivotal error $\epsilon=\epsilon_0$ and surrounding error $\epsilon_\mu=\epsilon_\mu^0$. Here we slightly abuse the notation $\epsilon_0$ and $\epsilon_\mu^0$. When applying OBD on $\mathcal{D}_\texttt{real}$, there is an additional distillation loss $\mu$. We assume that policies trained on the corresponding distilled dataset have $\epsilon = \epsilon_0 + \mu$ and $\epsilon_\mu  = \epsilon_\mu^0 + \mu $. To illustrate the effect of dataset characteristics on OBD, we compare two different original datasets: $\mathcal{D}_\texttt{real}^\texttt{piv}$ that excels in reducing pivotal error ($\epsilon_0=0.01$, $\epsilon_\mu^0 = 0.5$) and $\mathcal{D}_\texttt{real}^\texttt{surr}$ that performs well in minimizing surrounding error ($\epsilon_0=0.1$, $\epsilon_\mu^0=0.2$), simulating the \texttt{M-E} and \texttt{M-R} datasets, respectively. With a horizon length of $T=1000$, we analyze the policy gap bound under different values of $\mu$, as shown in Figure \ref{figure:thm error}. The results show that when $\mu$ is small, $\mathcal{D}_\texttt{real}^\texttt{piv}$ yields higher-quality synthetic data with a lower bound. However, when distillation loss cannot be effectively minimized (common in OBD), $\mathcal{D}_\texttt{real}^\texttt{surr}$ can produce better distilled data, {\it i.e.}, policies trained on the distilled data has a lower performance gap bound. To this end, the original dataset with diverse states, which benefits small surrounding error, is critical for OBD.


To remove the dependency on 
$\mathbb{E}_{s\sim d_\pi^\ast(s)}$ when computing pivotal error, we further derive the high probability bound by following the proof in \citep{mohri2018foundations} that mainly utilizes Hoeffding's inequality and the union bound.
\begin{corollary}
\label{corollary:1}
The expert dataset $\{(s_i, \pi^\ast(s_i))\}_{i=1}^m$ are i.i.d drawn from $d_{\pi^\ast}(s)$. Suppose $\pi^\ast$ and $\hat{\pi}$ are deterministic and the provided function class $\Pi$ has finite $|\Pi|$. If $\frac{1}{m}\sum_{i=1}^m \mathbb{I}\left(  \hat{\pi}(s_i) \neq \pi^*(s_i) \right)    \leq \epsilon$, then, with the probability of at least $1-\delta$  over a sample of size $m$, the following equality holds for all $ \pi \in \Pi:$ 
\begin{align}
\frac{\vert J(\pi^*) - J(\pi)\vert}{TR_{\max}} \leq \left(\epsilon_\mu T +3\right) \left(\epsilon + \sqrt{\frac{\log |\Pi|+\log \frac{2}{\delta}}{2 m}}\right) \nonumber
\end{align}
\end{corollary}

%% file: sec/density_weight_algorithm.tex
\begin{algorithm}[t]
\SetKwInOut{KwIn}{Input}
\SetKwInOut{KwOut}{Output}
\SetKw{KwBy}{by}
\caption{State density weighted OBD}
\label{alg:obd}
\KwIn{Offline RL dataset $\mathcal{D}_\texttt{off}$, synthetic data size $N_\texttt{syn}$, loop step $T_\texttt{in}$, $T_\texttt{out}$, learning rate $\alpha_0$, $\alpha_1$, momentum rate $\beta_0$, $\beta_1$, SDW intensity $\tau$}
\KwOut{Synthetic dataset $\mathcal{D}_\texttt{syn}$}
$\pi^*, q_{\pi^*} \leftarrow \texttt{OfflineRL}(\mathcal{D}_\texttt{off})$ \\
Initialize $\mathcal{D}_\texttt{syn}=\{(s_i, a_i)\}_{i=1}^{{N}_\texttt{syn}}$ by randomly sampling $(s_i, a_i)\sim \mathcal{D}_\texttt{off}$ \\
Construct $\mathcal{D}_\texttt{real}=\{(s_i, \pi^\ast(s_i))\}_{i=1}^{N_\texttt{off}}$ for $s_i \in \mathcal{D}_\texttt{off}$ \\
$d(s) \leftarrow \texttt{DensityEstimation}(\mathcal{D}_\texttt{real})$ \\
\For{$t_\texttt{out}=1$ \textbf{\text{to}} ${T}_\texttt{out}$}{
Randomly initialize policy network parameters $\theta_0$ \\
$\triangleright$ Behavioral cloning with synthetic data.\\
\For{$t=1$ \textbf{\text{to}} ${T}_\texttt{in}$}{
Compute the BC loss {\it w.r.t.} synthetic data $\mathcal{L}_{t-1} = \ell_\texttt{BC}(\theta_{t-1}, \mathcal{D}_\texttt{syn})$ \\
Update $\theta_t \leftarrow \texttt{GradDescent}(\nabla_{\theta_{t-1}} \mathcal{L}_{t-1}, \alpha_0, \beta_0)$
}
Sample a minibatch $\mathcal{B}=\{(s_i,a_i)\}_{i=1}^{|\mathcal{B}|}$ from $\mathcal{D}_\texttt{real}$ \\
Compute $\mathcal{H}(\pi_{\theta}, \mathcal{B}) = \frac{1}{\vert \mathcal{B} \vert} \sum_{i=1}^{|\mathcal{B}|} \frac{q_{\pi^*}(s_i, a_i)}{\blue{ d^\tau(s_i)}} \left(\pi_{\theta}\left(s_i\right) - a_i \right)^2$ \\
Update $\mathcal{D}_\texttt{syn} \leftarrow \texttt{GradDescent}(\nabla_{\mathcal{D}_\texttt{syn}} \mathcal{H}(\pi_{\theta},\mathcal{B}), \alpha_1, \beta_1)$ \\
}
\end{algorithm}

\begin{figure}[t]
\centering
\begin{subfigure}[b]{0.15\textwidth}
    \centering
    \includegraphics[width=\columnwidth]{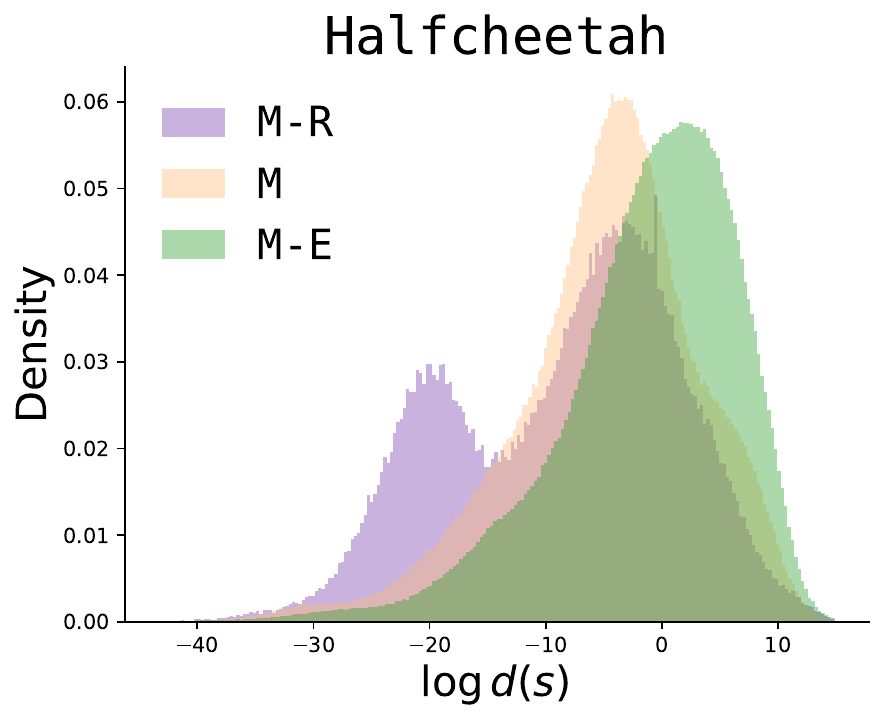}
    \label{figure:halfcheetah hist}
\end{subfigure}
\hfill
\begin{subfigure}[b]{0.15\textwidth}
    \centering
    \includegraphics[width=\columnwidth]{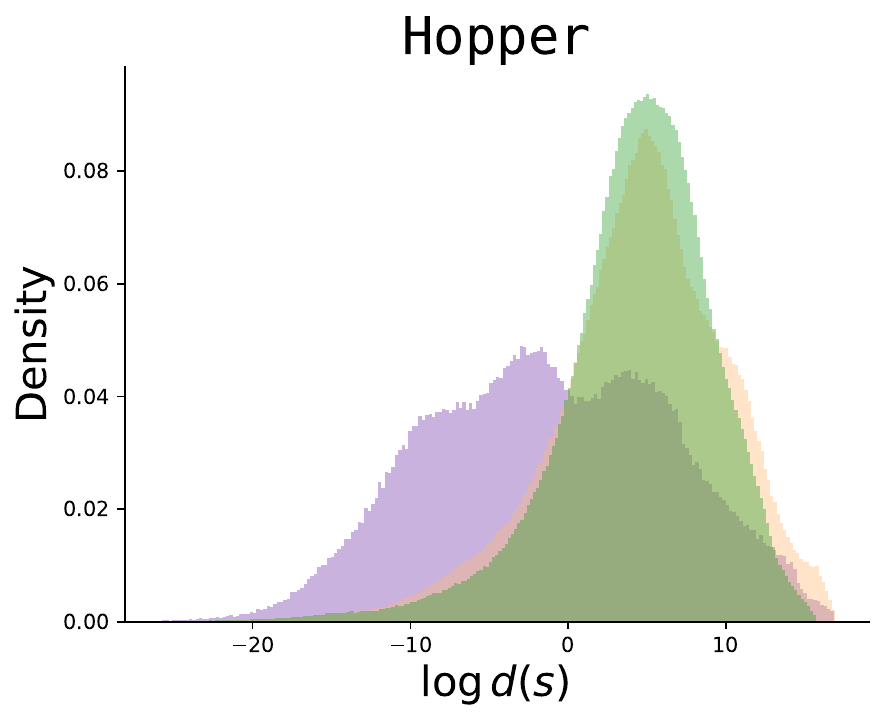}
    \label{figure:hopper hist}
\end{subfigure}
\hfill
\begin{subfigure}[b]{0.15\textwidth}
    \centering
    \includegraphics[width=\columnwidth]{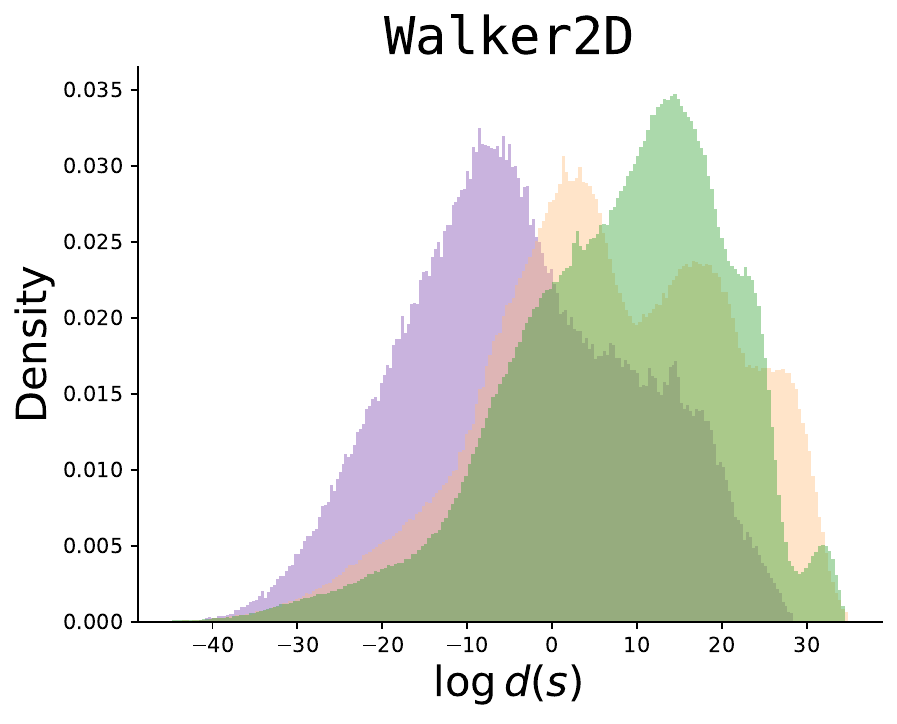}
    \label{figure:walker2d hist}
\end{subfigure}
\caption{The distribution of state density $d(s)$. The y-axis represents the density of the value $\log d(s)$, analogous to a probability density function.}
\label{figure:density hist}
\end{figure}

\begin{table*}[t]
\scriptsize
    \centering
    \caption{Offline behavior distillation performance on D4RL offline datasets. The result for Random Selection (Random) is obtained by repeating $10$ times. For DBC, PBC, and Av-PBC, the results are averaged across five seeds and the last five evaluation steps. The best OBD result for each dataset is marked with \textbf{bold} scores.}
    \setlength{\tabcolsep}{10pt} 
    \resizebox{\linewidth}{!}{
    \begin{tabular}{c  cc  cc  cc c}
    \toprule
     \multirow{2}{*}{Method}  & \multicolumn{2}{c}{Halfcheetach} & \multicolumn{2}{c}{Hopper} & \multicolumn{2}{c}{Walker2D} & \multirow{2}{*}{Average} \\
     & M & M-E & M & M-E & M & M-E \\
     \specialrule{0.7pt}{0.2\jot}{0.2pc}
     Rand ($\mathcal{D}_\texttt{off}$) &  $1.8$ & $2.0$  & $19.2$ & $11.6$ &  $4.9$ & $6.7$ & $7.7$   \\
     Rand ($\mathcal{D}_\texttt{real}$) &  $5.9$ & $7.8$  & $29.1$ & $27.1$ &  $17.1$ & $17.8$ & $17.5$   \\
     DBC &   $28.2$ & \bm{$29.0$} &   {$37.8$} & $31.1$ &   $29.3$ & $11.7$   & $27.9$ \\
     PBC &   $30.9$ & $20.5$ &   $25.1$ & $33.4$ &   $33.2$ & $34.0$   & $29.5$ \\
     Av-PBC &  {$36.9$} & $22.0$ &  $32.5$ & {$38.7$} &  {$39.5$} & {$42.1$}   & {$35.3$}\\
     SDW (ours) &  \bm{$39.5$} & $25.0$ &  \bm{$38.4$} & \bm{$42.6$} &  \bm{$42.5$} & \bm{$44.6$}   & \bm{$38.8$}\\
     \specialrule{0.5pt}{0.2\jot}{0.2pc}
     BC (Whole) &   $42.3$ & $59.8$ &   $50.2$ & $51.7$ &   $65.9$ & $89.6$ & $59.9$  \\
     OffRL (Whole) &   $47.6$ & $50.8$ &   $56.4$ & $107.3$ &  $84.0$ & $109.0$ & $75.6$  \\
    \bottomrule
    \end{tabular}
    }
    \label{tab:obd result}
\end{table*}
\section{State Density Weighted OBD}
\label{sec:exp}



The preceding analysis highlights the benefits of state diversity in OBD. To ensure that diverse state information is more evenly captured by distilled synthetic dataset, thereby decreasing surrounding error of policies trained on it, we revise the Av-PBC method by explicitly integrating state density probability into the OBD objective. Specifically, the state density for original dataset $\mathcal{D}_\texttt{real}$ is denoted as $d(s)$. Then, we leverage $\frac{1}{d(s)}$ to weight OBD objective, so that rare states in $\mathcal{D}_\texttt{real}$ receive greater attention during the distillation process. The resulting \textit{\textbf{s}tate \textbf{d}ensity \textbf{w}eighted} (\textbf{SDW}) OBD objective is presented below:
\begin{align}
    \mathcal{H}(\pi, \mathcal{D}_\texttt{real})= \mathbb{E}_{(s,a)\sim \mathcal{D}_\texttt{real}} \left[\frac{q_{\pi^*}(s,a)}{\blue{d^\tau(s)}} \left(\pi\left(s\right) - a  \right)^2 \right], \nonumber
\end{align}
where the hyperparameter $\tau \geq 0$ controls the intensity of density weighting: a larger $\tau$ places greater emphasis on rare states, while $\tau=0$ indicates that the absence of density weighting. The pseudo-code for SDW is in Algorithm \ref{alg:obd}. 

We utilize Masked Autoregressive Flow \citep{papamakarios2017masked} to estimate $d(s)$ for $s \in \mathcal{D}_\texttt{real}$, with the distribution of $\log d(s)$ depicted in Figure \ref{figure:density hist}. 
The plots reveal that states in \texttt{M} and \texttt{M-E} exhibits significantly higher $d(s)$ values compared to those in \texttt{M-R}, considering the logarithmic scale on the x-axis. Consequently, we primarily evaluate SDW on \texttt{M} and \texttt{M-E} data, as their limited state diversity leads to unsatisfactory OBD performance.

%% file: sec/experiments.tex
\begin{table*}[t]
\scriptsize
    \centering
    \caption{Offline behavior distillation performance across various policy network architectures and optimizers (Optim). {\color{red} Red-colored} scores and {\color{ForestGreen} green-colored} scores in brackets denote the performance degradation and improvement, respectively, compared to the default training setting. Better OBD result between SDW and Av-PBC for each cross scenario is marked with \textbf{bold} scores. The results are averaged over five random seeds and the last five evaluation steps.}
    \setlength{\tabcolsep}{7pt} 
    \resizebox{\linewidth}{!}{
    \begin{tabular}{c c  cc  cc  cc cc}
    \toprule
     & \multirow{2}{*}{Arch/Opt}  & \multicolumn{2}{c}{Halfcheetach} & \multicolumn{2}{c}{Hopper} & \multicolumn{2}{c}{Walker2D} & \multirow{2}{*}{\makecell{SDW \\ Average}} & \multirow{2}{*}{\makecell{Av-PBC \\ Average}} \\
     & & M & M-E & M & M-E & M & M-E  \\
     \specialrule{0.7pt}{0.2\jot}{0.2pc}
     \multirow{5}{*}{\rotatebox{90}{Architecture}} & $2$-layer & $37.4$ & $7.5$ & $25.9$ & $49.9$ & $33.3$ & $50.8$ & \redformat{\bm{34.1}}{4.7} & $33.4$\\
     &$3$-layer & $40.7$ & $21.9$ & $30.5$ & $43.9$ & $48.2$ & $55.3$ & \greenformat{\bm{40.1}}{1.3} & $37.9$  \\
     &$5$-layer & $39.3$ & $24.5$ & $33.6$ & $38.9$ & $41.2$ & $38.9$ & \redformat{\bm{36.1}}{2.7} & $34.6$ \\
     &$6$-layer & $37.0$ & $21.6$ & $32.8$ & $34.1$ & $35.7$ & $34.3$ & \redformat{\bm{32.7}}{6.1} & $27.5$ \\
     &Residual & $38.6$ & $20.5$ & $35.4$ & $43.4$ & $39.0$ & $41.2$ & \redformat{\bm{36.4}}{2.4} & $33.0$ \\
     \specialrule{0.5pt}{0.2\jot}{0.2pc}
     \multirow{3}{*}{\rotatebox{90}{Optim}}& Adam & $39.3$ & $24.5$ & $39.7$ & $41.6$ & $43.9$ & $43.1$ & \redformat{\bm{38.7}}{0.1} & $36.8$\\
     & AdamW & $39.1$ & $24.6$ & $37.4$ & $42.1$ & $44.8$ & $44.6$ & \grayformat{\bm{38.8}}{0.0} & $36.6$\\
     & SGDm & $39.7$ & $24.0$ & $37.1$ & $39.2$ & $44.1$ & $43.2$ & \redformat{\bm{37.9}}{0.9} & $35.1$ \\
    \bottomrule
\end{tabular}
    }
    \label{tab:cross generalization}
\end{table*}


\vspace{1mm}
\noindent \textbf{Setup \ \ }
The advanced offline RL algorithm of Cal-QL \citep{nakamoto2023calql} is employed to extract $\pi^*$ and $q_{\pi^*}$ from $\mathcal{D}_\texttt{off}$. 
$\tau$ is set to $0.1$ in SDW. We use a four-layer MLP as the default architecture of policy networks. The size of synthetic data $N_\texttt{syn}$ is $256$. Standard SGD is employed in both inner and outer optimization, and learning rates $\alpha_0=0.1$ and $\alpha_1=0.1$ for the inner and outer loop, respectively, and corresponding momentum rates $\beta_0=0$ and $\beta_1=0.9$. 
More details can be found in Appendix \ref{app:implementation details}.

\vspace{1mm}
\noindent \textbf{Evaluation \ \ } 
To evaluate $\mathcal{D}_\texttt{syn}$, we train policies on $\mathcal{D}_\texttt{syn}$ with 
BC, and obtain the averaged return by interacting with the environment for $10$ episodes. We use \texttt{\small{normalized return}} 
for better visualization: $\texttt{\small{normalized return}}=100 \times \frac{\texttt{return - random return}}{\texttt{expert return - random return}}$, where \texttt{\small{random return}} and \texttt{\small{expert return}} denote expected returns of random policies and the expert policy (online SAC \citep{haarnoja2018soft}), respectively.

\vspace{1mm}
\noindent \textbf{Baselines \ \ } 
(1) {\it Rand ($\mathcal{D}_\texttt{off}$)}: randomly selects $N_\texttt{syn}$ state-action pairs from $\mathcal{D}_\texttt{off}$; (2) {\it Rand ($\mathcal{D}_\texttt{real}$)}: randomly selects $N_\texttt{syn}$ state-action pairs from $\mathcal{D}_\texttt{real}$; (3) {\it Data-based BC (DBC)}: minimizes $\mathcal{H}_{\texttt{DBC}}=\mathbb{E}_{(s,a)\sim\mathcal{D}_\texttt{off}}\left[\left(\pi(s) - a\right)^2 \right]$; (4) {\it PBC}; (5) {\it Av-PBC}. 
We also report results of behavioral cloning and Cal-QL in terms of training on the whole offline dataset $\mathcal{D}_\texttt{off}$ for a comprehensive comparison.

\subsection{Main Results}
We compare our approach of SDW with other OBD algorithms (DBC, PBC, Av-PBC) across offline datasets of varying quality and environments, as shown in Table \ref{tab:obd result}. From the results, we have several observations: (1) SDW achieves the best performance among these OBD approaches; and (2) SDW consistently outperforms Av-PBC on all \texttt{M} and \texttt{M-E} dataset ($38.8$ \textit{vs.} $35.3$), which illustrate the effectiveness of SDW in mitigating poor state diversity in original data.

\vspace{1mm}
\noindent\textbf{Cross Architecture/Optimizer \ \ }
Although results in Table \ref{tab:obd result} are obtained by default configuration, we further assess synthetic data distilled by SDW across (1) multiple policy network architectures (2/3/5/6-layer and residual MLPs) and (2) different optimizers (Adam, AdamW, and SGDm with momentum=$0.9$). The results are shown in Table \ref{tab:cross generalization}, and we observe the following: (1) despite a slight decrease, synthetic datasets distilled by SDW remain effective for training various policy networks (SDW Average); (2) synthetic behavioral datasets exhibit greater robustness to variations in optimizers compared to differences in architectures; and (3) SDW consistently outperform Av-PBC across different training transfer scenarios (SDW Avg {\it vs.} Av-PBC Avg). Therefore, SDW-distilled datasets show satisfactory cross-architecture/optimizer performance.

\begin{figure}[t]
\centering
\begin{subfigure}{0.15\textwidth}
    \centering
    \includegraphics[width=\columnwidth]{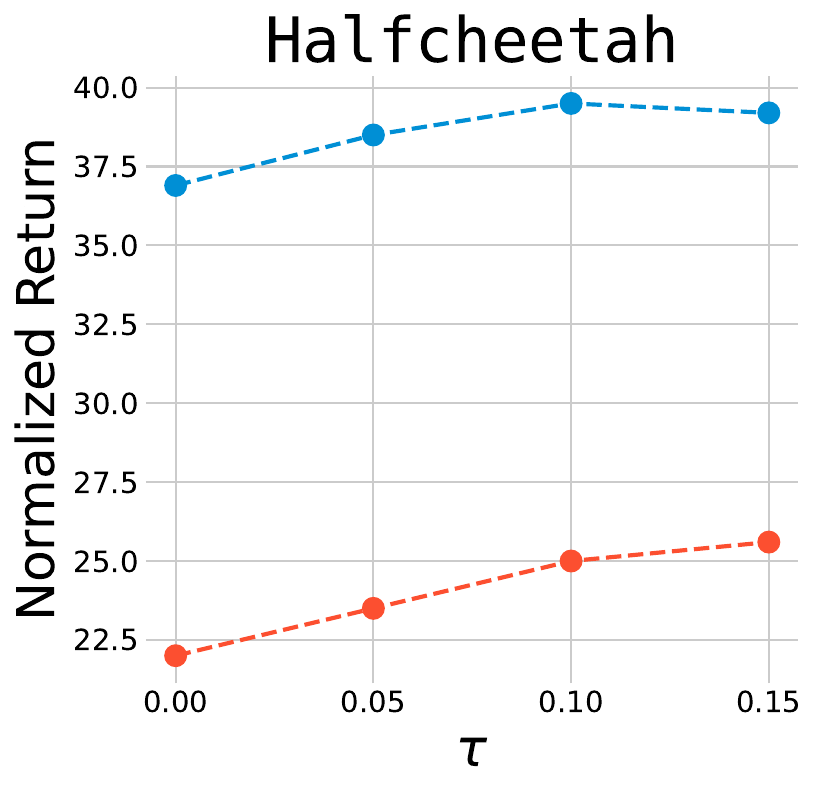}
    \label{figure:halfcheetah tau}   
\end{subfigure}
\hfill
\begin{subfigure}{0.15\textwidth}
    \centering
    \includegraphics[width=0.98\columnwidth]{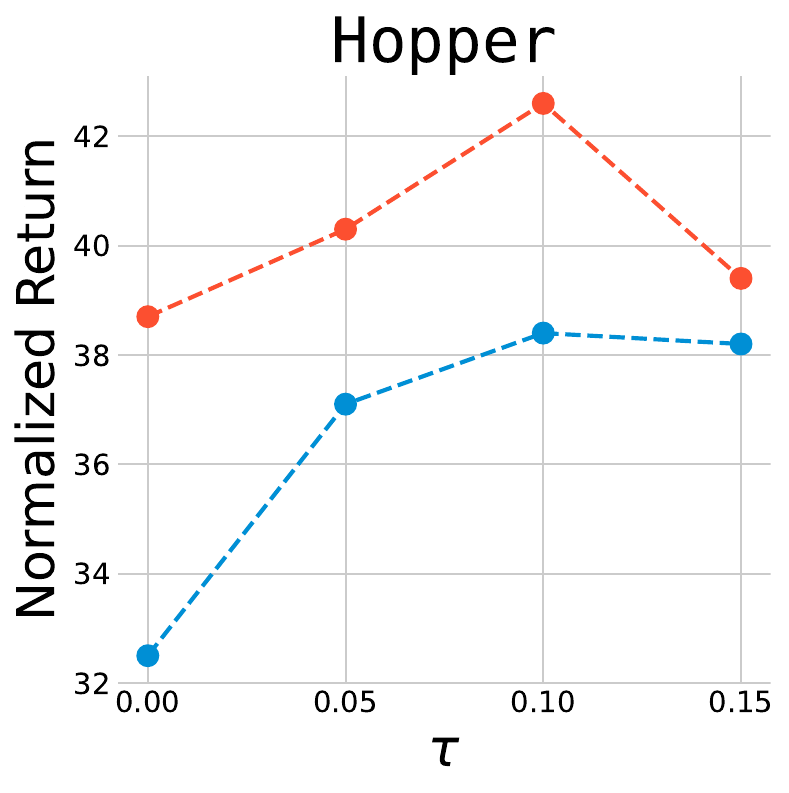}
    \label{figure:hopper tau}   
\end{subfigure}
\hfill
\begin{subfigure}{0.15\textwidth}
    \centering
    \includegraphics[width=0.98\columnwidth]{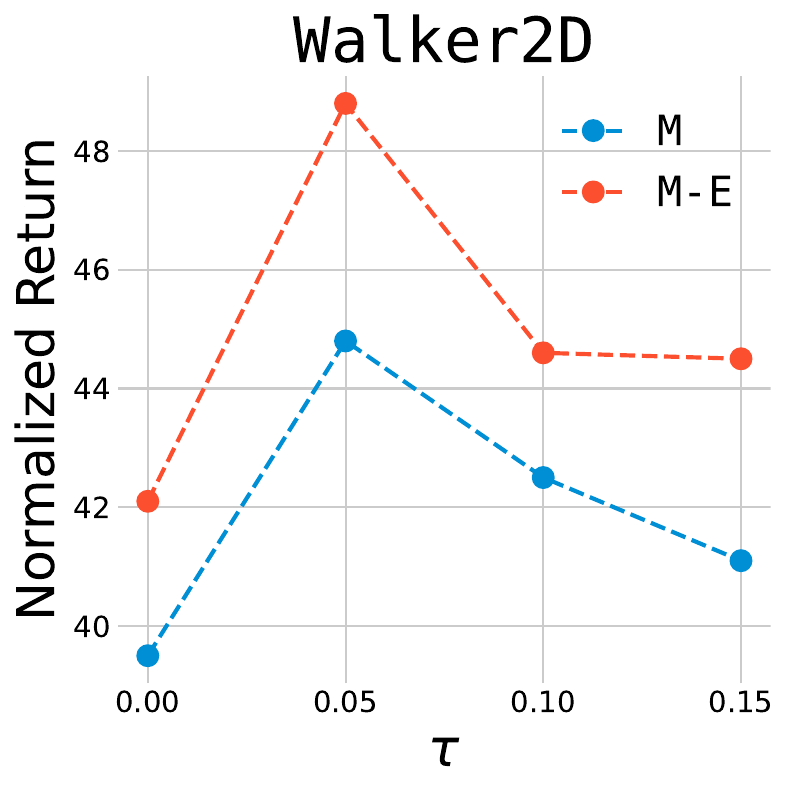}
    \label{figure:walker2d tau}   
\end{subfigure}
 \caption{Plots of SDW performance (normalized return) as functions of $\tau$. Each point is averaged on five trials.}
\label{figure:tau}
\end{figure}

\vspace{1mm}
\noindent\textbf{Hyperparamter Analysis \ \ }
The hyperparamter $\tau$, which determines the weight $\frac{1}{d^\tau(s)}$, is crucial to the performance of SDW. A larger $\tau$ assigns greater weights to rare states and is particularly beneficial when the original dataset exhibits limited state diversity. To investigate the impact of $\tau$, we conduct OBD using various $\tau$ ($0.05, 0.1, 0.15$). The results are presented in Figure \ref{figure:tau}. Our observations indicate that: (1) SDW with different $\tau$ values consistently outperform the baseline of $\tau=0$; and (2) difference datasets have distinct optimal $\tau$ values ($\tau=0.15$ for \texttt{Halfcheetah}, $\tau=0.1$ for \texttt{Hopper}, and $\tau=0.05$ for \texttt{Walker2D}). 

%% file: sec/discussion.tex
\section{Discussion and Limitation}
\label{sec:discussion}

\noindent\textbf{Discussion\ \ }
While data diversity is widely recognized as a cornerstone of modern model training, few studies have investigated how to balance improvements in diversity and data quality under constrained resource budgets. In this paper, we demonstrate that data diversity plays a more prominent role in the distillation setting compared to standard model training. This is largely due to the non-linear relationship between model performance and training loss: policies trained on more diverse datasets exhibit less performance degradation as training loss increases (Figure~\ref{figure:error_vs_return}). Consequently, a diverse dataset may be more beneficial for OBD, even if it appears inferior to high-quality datasets in direct policy learning.


\vspace{2mm}

\noindent\textbf{Limitations \ \ }
Our proposed method of state density weighted OBD requires the estimation of state density, making the performance of SDW dependent on the accuracy of density estimation. Moreover, offline RL datasets display varying levels of state diversity, necessitating different optimal values of $\tau$. For instance, datasets with lower state diversity intuitively require a larger $\tau$ to enhance the weighting effect. Therefore, how to automatically select a proper $\tau$ for specific original datasets is a promising area for future research.




\section{Conclusion}
In this paper, we underscore the essential role of state diversity in OBD. We begin by uncovering misalignment between original and distilled datasets: although policies trained on \texttt{M-E} (higher state quality) generally outperform those trained on \texttt{M-R} (greater state diversity), synthetic datasets distilled from \texttt{M-R} enable better policy training than those distilled from \texttt{M-E}. Through empirical analysis, we demonstrate that datasets with greater state diversity have improved resilience to significant training loss, contributing to this misalignment in OBD, where 
the training loss is challenging to minimize. By associating state quality and diversity in reducing pivotal and surrounding error, respectively, our theoretical investigation further elucidates that policy performance is influenced by both pivotal and surrounding errors, with the latter becoming increasingly important when pivotal error is pronounced, as is the case in OBD. 
We then propose a novel algorithm, SDW, which prioritizes state diversity by weighting the OBD loss according to state density. Extensive experiments on multiple D4RL datasets demonstrate that SDW substantially improves OBD performance, particularly in scenarios where original datasets exhibit limited state diversity.

%% file: sec/appendix.tex
\section{Proofs}
\label{app:proof}

\subsection{Proof of Theorem \ref{thm:expert}}
\label{app:proof thm1}
\begin{proof}
According to the expected return  $$J({\pi}) = \sum_{t=1}^T \mathbb{E}_{s\sim d_{{\pi}}^t, a\sim {\pi}}[r(s,a)],$$ we have
\begin{align}
    \vert & J(\pi^\ast) - J(\hat{\pi}) \vert  \nonumber \\
    & \leq \sum_{t=1}^T \left\vert \mathbb{E}_{s\sim d_{\pi^\ast}^t, a\sim \pi^\ast}[r(s,a)] - \mathbb{E}_{s\sim d_{\hat{\pi}}^t, a\sim \hat{\pi}}[r(s,a)] \right\vert \nonumber
\end{align}

For specific step $t$,
\begin{align}
    & \left\vert \mathbb{E}_{s\sim d_{\pi^\ast}^t, a\sim \pi^\ast}[r(s,a)] - \mathbb{E}_{s\sim d_{\hat{\pi}}^t, a\sim \hat{\pi}}[r(s,a)] \right\vert \nonumber \\
    &\leq  \underbrace{\left\vert \mathbb{E}_{s\sim d_{\pi^\ast}^t, a\sim \pi^\ast}[r(s,a)] \blue{- \mathbb{E}_{s\sim d_{\pi^\ast}^t, a\sim \hat{\pi}}[r(s,a)]}\right\vert}_{A(t)} \nonumber \\
    & \quad  \blue{+} \underbrace{\left\vert \blue{ \mathbb{E}_{s\sim d_{\pi^\ast}^t, a\sim \hat{\pi}}[r(s,a)]} - \mathbb{E}_{s\sim d_{\hat{\pi}}^t, a\sim \hat{\pi}}[r(s,a)]\right\vert}_{B(t)} \nonumber
\end{align}

For $A(t)$, we have
\begin{align}
    A(t) &= \left\vert \mathbb{E}_{s\sim d_{\pi^\ast}^t}[\sum_{a\in \mathcal{A}} (\pi^\ast(a|s) - \hat{\pi}(a|s))r(s,a)] \right\vert \nonumber \\
    &\leq \mathbb{E}_{s\sim d_{\pi^\ast}^t}[\sum_{a\in \mathcal{A}} \left\vert \pi^\ast(a|s) - \hat{\pi}(a|s)\right\vert r(s,a)]  \nonumber \\
    &\leq \mathbb{E}_{s\sim d_{\pi^\ast}^t}[\sum_{a\in \mathcal{A}} \left\vert \pi^\ast(a|s) - \hat{\pi}(a|s)\right\vert] R_{\max}  \nonumber \\
    &= \epsilon_t R_{\max} \nonumber
\end{align}
For $B(t)$, with $d_{\pi}^t(s) = \sum_{s',a' \in \mathcal{S}\times \mathcal{A}}d_{\pi}^{t-1}(s') \pi(a'|s') T(s|s',a')$ and $T(s|s',a') \in [0,1]$, we have
\begin{align}
\label{eq:bt}
    B(t) &= \left\vert \sum_{s\in\mathcal{S}} \left(d_{\pi^\ast}^t(s) - d_{\hat{\pi}}^t(s)\right) \mathbb{E}_{a\sim \hat{\pi}}[r(s,a)] \right\vert  \nonumber\\
    &\leq  R_{\max} \sum_{s\in\mathcal{S}} \left\vert d_{\pi^\ast}^t(s) - d_{\hat{\pi}}^t(s)\right\vert 
\end{align}
For the term $C(t)=\sum_{s\in\mathcal{S}} \left\vert d_{\pi^\ast}^t(s) - d_{\hat{\pi}}^t(s)\right\vert$, we have
\begin{align}
\label{eq:iterative}
    C(t) & = \sum_{s\in\mathcal{S}} \left\vert d_{\pi^\ast}^t(s) - d_{\hat{\pi}}^t(s)\right\vert \nonumber\\
    & \leq  \sum_{s',a' \in \mathcal{S}\times \mathcal{A}} \sum_{s\in\mathcal{S}}   \left\vert d_{\pi^\ast}^{t-1}(s') \pi^\ast(a'|s') - d_{\hat{\pi}}^{t-1}(s') \hat{\pi}(a'|s') \right\vert T(s|s',a')   \nonumber\\
    & =  \sum_{s,a \in \mathcal{S}\times \mathcal{A}} \left\vert d_{\pi^\ast}^{t-1}(s) \pi^\ast(a|s) - d_{\hat{\pi}}^{t-1}(s) \hat{\pi}(a|s) \right\vert \nonumber\\
    & =  \sum_{s,a \in \mathcal{S}\times \mathcal{A}} \left\vert d_{\pi^\ast}^{t-1}(s) \pi^\ast(a|s) \blue{ - d_{\pi^\ast}^{t-1}(s) \hat{\pi}(a|s) + d_{\pi^\ast}^{t-1}(s) \hat{\pi}(a|s)} \right.\nonumber \\
    & \qquad \qquad \qquad \qquad \qquad \qquad\qquad\qquad\qquad\quad \left. -  d_{\hat{\pi}}^{t-1}(s) \hat{\pi}(a|s) \right\vert \nonumber \\
    & \leq  ( \mathbb{E}_{s\sim d_{\pi^\ast}^{t-1}}[\sum_{a\in \mathcal{A}} \left\vert \pi^\ast(a|s) - \hat{\pi}(a|s) \right\vert] + \sum_{s \in \mathcal{S}} \left\vert d_{\pi^\ast}^{t-1}(s) - d_{\hat{\pi}}^{t-1}(s) \right\vert \nonumber \\
    & \leq \epsilon_{t-1}  +  \sum_{s \in \mathcal{S}} \left\vert d_{\pi^\ast}^{t-1}(s) - d_{\hat{\pi}}^{t-1}(s) \right\vert \nonumber \\
    & \leq  \epsilon_{t-1}  + C(t-1) 
\end{align}
Plugging the iterative formula of Eq. \ref{eq:iterative} into Eq. \ref{eq:bt} yields $B(t) \leq \sum_{i=1}^{t-1} \epsilon_i  R_{\max}$.
Therefore,
\begin{align}
    | J(\pi^\ast) - J(\hat{\pi})|  &\leq \sum_{t=1}^T A(t) + B(t) \nonumber \\
    & = \sum_{t=1}^T \sum_{i=1}^t \epsilon_i R_{\max} \nonumber \\
    & \leq T \sum_{t=1}^T \epsilon_t R_{\max} \nonumber \\
    & = \epsilon T^2 R_{\max}
\end{align}
The proof is completed.
\end{proof}

\subsection{Proof of Corollary \ref{corollary:1}}
\label{app:proof corollary}
We first derive the lemma that bounds the expected risk via empirical risk by following the proof of Theorem 2.2 in \citep{mohri2018foundations}.
\begin{lemma}
\label{lemma:bound}
The expert dataset $\{(s_i, \pi^\ast(s_i))\}_{i=1}^m$ are i.i.d. drawn from $d_{\pi^\ast}(s)$. Suppose $\pi^\ast$ and $\hat{\pi}$ are deterministic and the provided function class $\Pi$ has finite $|\Pi|$. Let $R(\pi) = \mathbb{E}_{s\sim d_{\pi^\ast}}\left[ \mathbb{I}\left(  \hat{\pi}(s) \neq \pi^*(s) \right)\right]$. If $\widehat{R}(\pi) = \frac{1}{m}\sum_{i=1}^m \mathbb{I}\left(  \hat{\pi}(s_i) \neq \pi^*(s_i) \right)    \leq \epsilon$, then, with the probability of at least $1-\delta$, the following inequality holds:
    \begin{equation}
        \forall \pi \in \Pi, \quad R(\pi) \leq \widehat{R}(\pi) + \sqrt{\frac{\log |\Pi|+\log \frac{2}{\delta}}{2 m}} . \nonumber 
    \end{equation}
\end{lemma}
\begin{proof}
The expected risk and empirical risk are defined as
\begin{align}
    & R(\pi) = \mathbb{E}_{(s)\sim d_{\pi^\ast}}\left[ \mathbb{I}\left(  \hat{\pi}(s) \neq \pi^*(s) \right)\right], \nonumber \\
    & \widehat{R}(\pi) = \frac{1}{m}\sum_{i=1}^m \mathbb{I}\left(  \hat{\pi}(s_i) \neq \pi^*(s_i) \right) . \nonumber
\end{align}
For all $s \in \mathcal{S}$, $  \mathbb{I}\left(  \hat{\pi}(s) \neq \pi^*(s) \right) \in [0, 1]$ holds. By applying Hoeffding's inequality, we have
\begin{equation}
\label{eq:single hypothesis}
    \Pr[\vert R(\pi) - \widehat{R}(\pi) \vert \geq \zeta] \leq 2\exp(-{ 2 m \zeta^2}).
\end{equation}
Let $\pi_1, \pi_2, \cdots, \pi_{|\Pi|}$ be the element of $|\Pi|$. Using the union bound and applying Eq. \ref{eq:single hypothesis} yield:
\begin{align}
    &\Pr[\exists \pi |R(\pi) - \widehat{R}(\pi)|>\zeta] \nonumber \\
    &= \Pr[(|R(\pi_1) - \widehat{R}(\pi_1)|>\zeta) \vee ... \vee (|R(\pi_{|\Pi|}) - \widehat{R}(\pi_{|\Pi|})|>\zeta)] \nonumber \\
    &\leq \sum_{\pi\in \Pi} \Pr[|R(\pi) - \widehat{R}(\pi)|>\zeta] \nonumber \\
    &\leq 2|\Pi|\exp(-{ 2 m \zeta^2}). \nonumber
\end{align}
Setting the right-hand side to be equal to $\delta$ complete the proof.
\end{proof}
Combining Theorem \ref{thm:3} and Lemma \ref{lemma:bound}  completes the proof.

\begin{table*}[ht]
\scriptsize
    \centering
    \caption{Offline behavior distillation performance on D4RL offline datasets.}
    \vspace{-0.1cm}
    \setlength{\tabcolsep}{10pt} 
    \resizebox{0.8\linewidth}{!}{
    \begin{tabular}{c  cc  cc  cc c}
    \toprule
     \multirow{2}{*}{Method}  & \multicolumn{2}{c}{Halfcheetach} & \multicolumn{2}{c}{Hopper} & \multicolumn{2}{c}{Walker2D} & \multirow{2}{*}{Average} \\
     & M & M-E & M & M-E & M & M-E \\
     \specialrule{0.7pt}{0.2\jot}{0.2pc}
     Top Reward &  $1.6$ & $0.2$  & $6.1$ & $3.8$ &  $1.4$ & $1.5$ & $2.4$   \\
     Top q-value &  $1.8$ & $3.8$  & $7.0$ & $8.3$ &  $6.6$ & $4.9$ & $5.4$   \\
    \bottomrule
    \end{tabular}
    }
    \label{tab:more result}
\end{table*}

\begin{table*}[ht]
\scriptsize
    \centering
    \caption{The size and required training steps for convergence for each offline dataset. M denotes the million for simplicity. The size and step for synthetic data (Synset) are listed in the last column.}
    \resizebox{0.8\linewidth}{!}{
    \begin{tabular}{c  ccc  ccc  ccc c}
    \toprule
     \multirow{2}{*}{ }  & \multicolumn{3}{c}{Halfcheetach} & \multicolumn{3}{c}{Hopper} & \multicolumn{3}{c}{Walker2D} & \multirow{2}{*}{Synset}\\
     & M-R & M & M-E & M-R & M & M-E & M-R & M & M-E \\
     \specialrule{1pt}{0.2\jot}{0.2pc}
     
     Size & $0.2$M & $1$M & $2$M & $0.4$M &  $1$M & $2$M & $0.3$M &  $1$M & $2$M & $256$ \\
     \specialrule{0.3pt}{0.2\jot}{0.2pc}
     Step (k) & $40$ & $25$ & $100$ & $80$ & $50$ & $100$ & $60$ & $50$ & $125$  & $0.1$ \\
    \bottomrule
    \end{tabular}
    }
    \label{tab:training steps}
\end{table*}

\section{Implementation Details}
\label{app:implementation details}

This section provides all the additional implementation details of our experiments.

\paragraph{Normalized Loss vs. Normalized Return (Figure \ref{figure:error_vs_return}).} For policies trained on the original datasets, the total training epoch is $50,000$. The policy return and loss are evaluated during the training process. For policies trained on distilled datasets, the training epoch is $100$, and policies are only evaluated when the training is finished. The distilled datasets are generated using Av-PBC \citep{lei2024offline} every $1,000$ distillation steps. For specific \texttt{M-R} or \texttt{M-E} cases, after obtaining the training loss of policies trained on both the original and corresponding distilled datasets, the normalized loss is calculated by subtracting the minimum value among these losses. This normalization ensures that the \texttt{M-R} and \texttt{M-E} cases have the same convergence loss, enabling a fair comparison between \texttt{M-R} and \texttt{M-E}.


\paragraph{OBD Settings} The policy network is a 4-layer multilayer perceptron (MLP) with a width of $256$.  The synthetic data are initialized by randomly selecting $N_\texttt{syn}$ state-action pairs from the offline data. For DBC and PBC, the distillation step $T_\texttt{out}$ is set to $200$k for \texttt{Halfcheetah} and \texttt{Walker2D} and $50$k for \texttt{Hopper}, respectively. For Av-PBC and SDW, the distillation step $T_\texttt{out}$ is set to $50$k for \texttt{Halfcheetah} and \texttt{Walker2D} and $20$k for \texttt{Hopper}, respectively. The inner loop step $T_\texttt{in}$ is set to $100$.


\paragraph{Offline RL Policy Training} We use the advanced offline RL algorithm of Cal-QL \citep{nakamoto2023calql} to extract the expert policy $\pi^*$ and corresponding q value function $q_{\pi^*}$ from suboptimal offline data, and the implementation in \citep{tarasov2022corl} is employed in our experiments with default hyper-parameter setting.

\paragraph{Cross-architecture Experiments.} The width of MLPs are both $256$. The residual MLP is a 4-layer MLP, and the intermediate layers are packaged into the residual block.

\section{More Performance Comparison}
We selected a small subset of data with the same budget of $256$ based on reward (Top Reward) and q-value (Top q-value). The results, presented in the Table \ref{tab:more result}, show that both Top Reward and Top q-value perform significantly worse than random selection, indicating that simply filtering behavioral data based on reward or q-value may not be effective. A possible explanation is that selecting data based on a single criterion, such as reward or q-value, overlooks important factors like state diversity and long-horizon returns. In contrast, our OBD algorithm, SDW, effectively addresses these issues and synthesizes a compact yet informative behavioral dataset.

\section{Training Time Comparison}
\label{app:training time}
For the whole original data, offline RL algorithms require dozens of hours. Therefore, we solely compare the training time of BC on synthetic data and BC on original data. Because the training time varies with GPU models (NVIDIA V100 used in our experiments), we report the optimization step, which has a linear relationship to training time, required for training convergence for each original dataset, as shown in Table \ref{tab:training steps}.

%% file: bare_adv.bbl
\begin{thebibliography}{60}
\providecommand{\natexlab}[1]{#1}
\providecommand{\url}[1]{#1}
\csname url@samestyle\endcsname
\providecommand{\newblock}{\relax}
\providecommand{\bibinfo}[2]{#2}
\providecommand{\BIBentrySTDinterwordspacing}{\spaceskip=0pt\relax}
\providecommand{\BIBentryALTinterwordstretchfactor}{4}
\providecommand{\BIBentryALTinterwordspacing}{\spaceskip=\fontdimen2\font plus
\BIBentryALTinterwordstretchfactor\fontdimen3\font minus \fontdimen4\font\relax}
\providecommand{\BIBforeignlanguage}[2]{{%
\expandafter\ifx\csname l@#1\endcsname\relax
\typeout{** WARNING: IEEEtranN.bst: No hyphenation pattern has been}%
\typeout{** loaded for the language `#1'. Using the pattern for}%
\typeout{** the default language instead.}%
\else
\language=\csname l@#1\endcsname
\fi
#2}}
\providecommand{\BIBdecl}{\relax}
\BIBdecl

\bibitem[Levine et~al.(2020)Levine, Kumar, Tucker, and Fu]{levine2020offline}
S.~Levine, A.~Kumar, G.~Tucker, and J.~Fu, ``Offline reinforcement learning: Tutorial, review, and perspectives on open problems,'' \emph{arXiv preprint arXiv:2005.01643}, 2020.

\bibitem[Weerakoon et~al.(2024)Weerakoon, Sathyamoorthy, Elnoor, and Manocha]{10610132}
K.~Weerakoon, A.~J. Sathyamoorthy, M.~Elnoor, and D.~Manocha, ``Vapor: Legged robot navigation in unstructured outdoor environments using offline reinforcement learning,'' in \emph{2024 IEEE International Conference on Robotics and Automation (ICRA)}, 2024, pp. 10\,344--10\,350.

\bibitem[Snell et~al.(2023)Snell, Kostrikov, Su, Yang, and Levine]{snell2023offline}
\BIBentryALTinterwordspacing
C.~V. Snell, I.~Kostrikov, Y.~Su, S.~Yang, and S.~Levine, ``Offline {RL} for natural language generation with implicit language q learning,'' in \emph{The Eleventh International Conference on Learning Representations}, 2023. [Online]. Available: \url{https://openreview.net/forum?id=aBH_DydEvoH}
\BIBentrySTDinterwordspacing

\bibitem[Fu et~al.(2020)Fu, Kumar, Nachum, Tucker, and Levine]{fu2020d4rl}
J.~Fu, A.~Kumar, O.~Nachum, G.~Tucker, and S.~Levine, ``D4rl: Datasets for deep data-driven reinforcement learning,'' 2020.

\bibitem[Lei et~al.(2024)Lei, Zhang, and Tao]{lei2024offline}
\BIBentryALTinterwordspacing
S.~Lei, S.~Zhang, and D.~Tao, ``Offline behavior distillation,'' in \emph{The Thirty-eighth Annual Conference on Neural Information Processing Systems}, 2024. [Online]. Available: \url{https://openreview.net/forum?id=89fSR2gpxp}
\BIBentrySTDinterwordspacing

\bibitem[Gai et~al.(2023)Gai, Wang, and He]{gai2023offline}
S.~Gai, D.~Wang, and L.~He, ``Offline experience replay for continual offline reinforcement learning,'' \emph{arXiv preprint arXiv:2305.13804}, 2023.

\bibitem[Yu et~al.(2021{\natexlab{a}})Yu, Kumar, Chebotar, Hausman, Levine, and Finn]{yu2021conservative}
T.~Yu, A.~Kumar, Y.~Chebotar, K.~Hausman, S.~Levine, and C.~Finn, ``Conservative data sharing for multi-task offline reinforcement learning,'' \emph{Advances in Neural Information Processing Systems}, vol.~34, pp. 11\,501--11\,516, 2021.

\bibitem[Goecks et~al.(2019)Goecks, Gremillion, Lawhern, Valasek, and Waytowich]{goecks2019integrating}
V.~G. Goecks, G.~M. Gremillion, V.~J. Lawhern, J.~Valasek, and N.~R. Waytowich, ``Integrating behavior cloning and reinforcement learning for improved performance in dense and sparse reward environments,'' \emph{arXiv preprint arXiv:1910.04281}, 2019.

\bibitem[Zhao et~al.(2022)Zhao, Boney, Ilin, Kannala, and Pajarinen]{zhao2022adaptive}
Y.~Zhao, R.~Boney, A.~Ilin, J.~Kannala, and J.~Pajarinen, ``Adaptive behavior cloning regularization for stable offline-to-online reinforcement learning,'' \emph{arXiv preprint arXiv:2210.13846}, 2022.

\bibitem[Qiao and Wang(2023)]{qiao2023offline}
\BIBentryALTinterwordspacing
D.~Qiao and Y.-X. Wang, ``Offline reinforcement learning with differential privacy,'' in \emph{Thirty-seventh Conference on Neural Information Processing Systems}, 2023. [Online]. Available: \url{https://openreview.net/forum?id=YVMc3KiWBQ}
\BIBentrySTDinterwordspacing

\bibitem[Lange et~al.(2012)Lange, Gabel, and Riedmiller]{lange2012batch}
S.~Lange, T.~Gabel, and M.~Riedmiller, ``Batch reinforcement learning,'' in \emph{Reinforcement learning: State-of-the-art}.\hskip 1em plus 0.5em minus 0.4em\relax Springer, 2012, pp. 45--73.

\bibitem[Fujimoto and Gu(2021)]{fujimoto2021a}
\BIBentryALTinterwordspacing
S.~Fujimoto and S.~Gu, ``A minimalist approach to offline reinforcement learning,'' in \emph{Advances in Neural Information Processing Systems}, A.~Beygelzimer, Y.~Dauphin, P.~Liang, and J.~W. Vaughan, Eds., 2021. [Online]. Available: \url{https://openreview.net/forum?id=Q32U7dzWXpc}
\BIBentrySTDinterwordspacing

\bibitem[Tarasov et~al.(2024)Tarasov, Kurenkov, Nikulin, and Kolesnikov]{tarasov2024revisiting}
D.~Tarasov, V.~Kurenkov, A.~Nikulin, and S.~Kolesnikov, ``Revisiting the minimalist approach to offline reinforcement learning,'' \emph{Advances in Neural Information Processing Systems}, vol.~36, 2024.

\bibitem[Kumar et~al.(2019)Kumar, Fu, Soh, Tucker, and Levine]{kumar2019stabilizing}
A.~Kumar, J.~Fu, M.~Soh, G.~Tucker, and S.~Levine, ``Stabilizing off-policy q-learning via bootstrapping error reduction,'' \emph{Advances in neural information processing systems}, vol.~32, 2019.

\bibitem[Wu et~al.(2019)Wu, Tucker, and Nachum]{wu2019behavior}
Y.~Wu, G.~Tucker, and O.~Nachum, ``Behavior regularized offline reinforcement learning,'' \emph{arXiv preprint arXiv:1911.11361}, 2019.

\bibitem[Fujimoto et~al.(2019)Fujimoto, Meger, and Precup]{fujimoto2019off}
S.~Fujimoto, D.~Meger, and D.~Precup, ``Off-policy deep reinforcement learning without exploration,'' in \emph{International conference on machine learning}.\hskip 1em plus 0.5em minus 0.4em\relax PMLR, 2019, pp. 2052--2062.

\bibitem[Kostrikov et~al.(2022)Kostrikov, Nair, and Levine]{kostrikov2022offline}
\BIBentryALTinterwordspacing
I.~Kostrikov, A.~Nair, and S.~Levine, ``Offline reinforcement learning with implicit q-learning,'' in \emph{International Conference on Learning Representations}, 2022. [Online]. Available: \url{https://openreview.net/forum?id=68n2s9ZJWF8}
\BIBentrySTDinterwordspacing

\bibitem[Kumar et~al.(2020)Kumar, Zhou, Tucker, and Levine]{kumar2020conservative}
A.~Kumar, A.~Zhou, G.~Tucker, and S.~Levine, ``Conservative q-learning for offline reinforcement learning,'' \emph{Advances in Neural Information Processing Systems}, vol.~33, pp. 1179--1191, 2020.

\bibitem[Yu et~al.(2021{\natexlab{b}})Yu, Kumar, Rafailov, Rajeswaran, Levine, and Finn]{yu2021combo}
T.~Yu, A.~Kumar, R.~Rafailov, A.~Rajeswaran, S.~Levine, and C.~Finn, ``Combo: Conservative offline model-based policy optimization,'' \emph{Advances in neural information processing systems}, vol.~34, pp. 28\,954--28\,967, 2021.

\bibitem[Bai et~al.(2022)Bai, Wang, Yang, Deng, Garg, Liu, and Wang]{bai2022pessimistic}
\BIBentryALTinterwordspacing
C.~Bai, L.~Wang, Z.~Yang, Z.-H. Deng, A.~Garg, P.~Liu, and Z.~Wang, ``Pessimistic bootstrapping for uncertainty-driven offline reinforcement learning,'' in \emph{International Conference on Learning Representations}, 2022. [Online]. Available: \url{https://openreview.net/forum?id=Y4cs1Z3HnqL}
\BIBentrySTDinterwordspacing

\bibitem[Rigter et~al.(2022)Rigter, Lacerda, and Hawes]{rigter2022rambo}
M.~Rigter, B.~Lacerda, and N.~Hawes, ``Rambo-rl: Robust adversarial model-based offline reinforcement learning,'' \emph{Advances in neural information processing systems}, vol.~35, pp. 16\,082--16\,097, 2022.

\bibitem[Nakamoto et~al.(2023)Nakamoto, Zhai, Singh, Mark, Ma, Finn, Kumar, and Levine]{nakamoto2023calql}
\BIBentryALTinterwordspacing
M.~Nakamoto, Y.~Zhai, A.~Singh, M.~S. Mark, Y.~Ma, C.~Finn, A.~Kumar, and S.~Levine, ``Cal-{QL}: Calibrated offline {RL} pre-training for efficient online fine-tuning,'' in \emph{Thirty-seventh Conference on Neural Information Processing Systems}, 2023. [Online]. Available: \url{https://openreview.net/forum?id=GcEIvidYSw}
\BIBentrySTDinterwordspacing

\bibitem[Uehara and Sun(2022)]{uehara2022pessimistic}
\BIBentryALTinterwordspacing
M.~Uehara and W.~Sun, ``Pessimistic model-based offline reinforcement learning under partial coverage,'' in \emph{International Conference on Learning Representations}, 2022. [Online]. Available: \url{https://openreview.net/forum?id=tyrJsbKAe6}
\BIBentrySTDinterwordspacing

\bibitem[Kidambi et~al.(2020)Kidambi, Rajeswaran, Netrapalli, and Joachims]{kidambi2020morel}
R.~Kidambi, A.~Rajeswaran, P.~Netrapalli, and T.~Joachims, ``Morel: Model-based offline reinforcement learning,'' \emph{Advances in neural information processing systems}, vol.~33, pp. 21\,810--21\,823, 2020.

\bibitem[Yu et~al.(2020)Yu, Thomas, Yu, Ermon, Zou, Levine, Finn, and Ma]{yu2020mopo}
T.~Yu, G.~Thomas, L.~Yu, S.~Ermon, J.~Y. Zou, S.~Levine, C.~Finn, and T.~Ma, ``Mopo: Model-based offline policy optimization,'' \emph{Advances in Neural Information Processing Systems}, vol.~33, pp. 14\,129--14\,142, 2020.

\bibitem[Sachdeva and McAuley(2023)]{sachdeva2023data}
\BIBentryALTinterwordspacing
N.~Sachdeva and J.~McAuley, ``Data distillation: A survey,'' \emph{Transactions on Machine Learning Research}, 2023, survey Certification. [Online]. Available: \url{https://openreview.net/forum?id=lmXMXP74TO}
\BIBentrySTDinterwordspacing

\bibitem[Yu et~al.(2024)Yu, Liu, and Wang]{yu2024dataset}
R.~Yu, S.~Liu, and X.~Wang, ``Dataset distillation: A comprehensive review,'' \emph{IEEE Transactions on Pattern Analysis and Machine Intelligence}, vol.~46, no.~01, pp. 150--170, jan 2024.

\bibitem[Lei and Tao(2024)]{lei2024comprehensive}
S.~Lei and D.~Tao, ``A comprehensive survey of dataset distillation,'' \emph{IEEE Transactions on Pattern Analysis and Machine Intelligence}, vol.~46, no.~01, pp. 17--32, jan 2024.

\bibitem[Wang et~al.(2018)Wang, Zhu, Torralba, and Efros]{wang2018dataset}
T.~Wang, J.-Y. Zhu, A.~Torralba, and A.~A. Efros, ``Dataset distillation,'' \emph{arXiv preprint arXiv:1811.10959}, 2018.

\bibitem[Deng and Russakovsky(2022)]{deng2022remember}
\BIBentryALTinterwordspacing
Z.~Deng and O.~Russakovsky, ``Remember the past: Distilling datasets into addressable memories for neural networks,'' in \emph{Advances in Neural Information Processing Systems}, A.~H. Oh, A.~Agarwal, D.~Belgrave, and K.~Cho, Eds., 2022. [Online]. Available: \url{https://openreview.net/forum?id=RYZyj_wwgfa}
\BIBentrySTDinterwordspacing

\bibitem[Zhao et~al.(2021)Zhao, Mopuri, and Bilen]{zhao2021dataset}
\BIBentryALTinterwordspacing
B.~Zhao, K.~R. Mopuri, and H.~Bilen, ``Dataset condensation with gradient matching,'' in \emph{International Conference on Learning Representations}, 2021. [Online]. Available: \url{https://openreview.net/forum?id=mSAKhLYLSsl}
\BIBentrySTDinterwordspacing

\bibitem[Zhao and Bilen(2021)]{zhao2021dsa}
B.~Zhao and H.~Bilen, ``Dataset condensation with differentiable siamese augmentation,'' in \emph{International Conference on Machine Learning}.\hskip 1em plus 0.5em minus 0.4em\relax PMLR, 2021, pp. 12\,674--12\,685.

\bibitem[Zhao and Bilen(2023)]{zhao2023distribution}
------, ``Dataset condensation with distribution matching,'' in \emph{Proceedings of the IEEE/CVF Winter Conference on Applications of Computer Vision (WACV)}, 2023.

\bibitem[Wang et~al.(2022)Wang, Zhao, Peng, Zhu, Yang, Wang, Huang, Bilen, Wang, and You]{wang2022cafe}
K.~Wang, B.~Zhao, X.~Peng, Z.~Zhu, S.~Yang, S.~Wang, G.~Huang, H.~Bilen, X.~Wang, and Y.~You, ``Cafe: Learning to condense dataset by aligning features,'' in \emph{Proceedings of the IEEE/CVF Conference on Computer Vision and Pattern Recognition}, 2022, pp. 12\,196--12\,205.

\bibitem[Cazenavette et~al.(2022)Cazenavette, Wang, Torralba, Efros, and Zhu]{cazenavette2022dataset}
G.~Cazenavette, T.~Wang, A.~Torralba, A.~A. Efros, and J.-Y. Zhu, ``Dataset distillation by matching training trajectories,'' in \emph{Proceedings of the IEEE/CVF Conference on Computer Vision and Pattern Recognition}, 2022.

\bibitem[Cui et~al.(2023)Cui, Wang, Si, and Hsieh]{cui2023scaling}
J.~Cui, R.~Wang, S.~Si, and C.-J. Hsieh, ``Scaling up dataset distillation to imagenet-1k with constant memory,'' in \emph{Proceedings of the International Conference on Machine Learning (ICML)}, 2023.

\bibitem[Maekawa et~al.(2023)Maekawa, Kobayashi, Funakoshi, and Okumura]{maekawa2023dataset}
\BIBentryALTinterwordspacing
A.~Maekawa, N.~Kobayashi, K.~Funakoshi, and M.~Okumura, ``Dataset distillation with attention labels for fine-tuning {BERT},'' in \emph{Proceedings of the 61st Annual Meeting of the Association for Computational Linguistics (Volume 2: Short Papers)}, A.~Rogers, J.~Boyd-Graber, and N.~Okazaki, Eds.\hskip 1em plus 0.5em minus 0.4em\relax Toronto, Canada: Association for Computational Linguistics, Jul. 2023, pp. 119--127. [Online]. Available: \url{https://aclanthology.org/2023.acl-short.12/}
\BIBentrySTDinterwordspacing

\bibitem[Tao et~al.(2024)Tao, Kong, Kan, and Callot]{tao2024textual}
\BIBentryALTinterwordspacing
Y.~Tao, L.~Kong, A.~Kan, and L.~Callot, ``Textual dataset distillation via language model embedding,'' in \emph{Findings of the Association for Computational Linguistics: EMNLP 2024}, Y.~Al-Onaizan, M.~Bansal, and Y.-N. Chen, Eds.\hskip 1em plus 0.5em minus 0.4em\relax Miami, Florida, USA: Association for Computational Linguistics, Nov. 2024, pp. 12\,557--12\,569. [Online]. Available: \url{https://aclanthology.org/2024.findings-emnlp.733/}
\BIBentrySTDinterwordspacing

\bibitem[Wang et~al.(2024)Wang, Xu, Lu, and Li]{wang2024dancing}
Z.~Wang, Y.~Xu, C.~Lu, and Y.-L. Li, ``Dancing with still images: Video distillation via static-dynamic disentanglement,'' in \emph{Proceedings of the IEEE/CVF Conference on Computer Vision and Pattern Recognition}, 2024, pp. 6296--6304.

\bibitem[Ding et~al.(2025)Ding, Chen, and Yao]{ding2025condensing}
G.~Ding, R.~Chen, and A.~Yao, ``Condensing action segmentation datasets via generative network inversion,'' in \emph{Proceedings of the Computer Vision and Pattern Recognition Conference}, 2025, pp. 17\,733--17\,742.

\bibitem[Jin et~al.(2022{\natexlab{a}})Jin, Zhao, Zhang, Liu, Tang, and Shah]{jin2022graph}
\BIBentryALTinterwordspacing
W.~Jin, L.~Zhao, S.~Zhang, Y.~Liu, J.~Tang, and N.~Shah, ``Graph condensation for graph neural networks,'' in \emph{International Conference on Learning Representations}, 2022. [Online]. Available: \url{https://openreview.net/forum?id=WLEx3Jo4QaB}
\BIBentrySTDinterwordspacing

\bibitem[Jin et~al.(2022{\natexlab{b}})Jin, Tang, Jiang, Li, Zhang, Tang, and Yin]{jin2022condensing}
W.~Jin, X.~Tang, H.~Jiang, Z.~Li, D.~Zhang, J.~Tang, and B.~Yin, ``Condensing graphs via one-step gradient matching,'' in \emph{Proceedings of the 28th ACM SIGKDD Conference on Knowledge Discovery and Data Mining}, 2022, pp. 720--730.

\bibitem[Xu et~al.(2023)Xu, Chen, Pan, Chen, Das, Yang, and Tong]{xu2023kernel}
Z.~Xu, Y.~Chen, M.~Pan, H.~Chen, M.~Das, H.~Yang, and H.~Tong, ``Kernel ridge regression-based graph dataset distillation,'' in \emph{Proceedings of the 29th ACM SIGKDD Conference on Knowledge Discovery and Data Mining}, 2023, pp. 2850--2861.

\bibitem[Zheng et~al.(2023)Zheng, Zhang, Chen, Nguyen, Zhu, and Pan]{zheng2023structurefree}
\BIBentryALTinterwordspacing
X.~Zheng, M.~Zhang, C.~Chen, Q.~V.~H. Nguyen, X.~Zhu, and S.~Pan, ``Structure-free graph condensation: From large-scale graphs to condensed graph-free data,'' in \emph{Thirty-seventh Conference on Neural Information Processing Systems}, 2023. [Online]. Available: \url{https://openreview.net/forum?id=XkcufOcgUc}
\BIBentrySTDinterwordspacing

\bibitem[Liu et~al.(2024)Liu, Hao, Zheng, and Yu]{liu2024dataset}
\BIBentryALTinterwordspacing
Z.~Liu, K.~Hao, G.~Zheng, and Y.~Yu, ``Dataset condensation for time series classification via dual domain matching,'' in \emph{Proceedings of the 30th ACM SIGKDD Conference on Knowledge Discovery and Data Mining}, ser. KDD '24.\hskip 1em plus 0.5em minus 0.4em\relax New York, NY, USA: Association for Computing Machinery, 2024, p. 1980–1991. [Online]. Available: \url{https://doi.org/10.1145/3637528.3671675}
\BIBentrySTDinterwordspacing

\bibitem[Ding et~al.(2024)Ding, Liu, Zheng, Jin, and Kong]{ding2024condtsf}
\BIBentryALTinterwordspacing
J.~Ding, Z.~Liu, G.~Zheng, H.~Jin, and L.~Kong, ``Cond{TSF}: One-line plugin of dataset condensation for time series forecasting,'' in \emph{The Thirty-eighth Annual Conference on Neural Information Processing Systems}, 2024. [Online]. Available: \url{https://openreview.net/forum?id=L1jajNWON5}
\BIBentrySTDinterwordspacing

\bibitem[Miao et~al.(2024)Miao, Liu, Zhao, Guo, Yang, Zheng, and Jensen]{miao2024less}
H.~Miao, Z.~Liu, Y.~Zhao, C.~Guo, B.~Yang, K.~Zheng, and C.~S. Jensen, ``Less is more: Efficient time series dataset condensation via two-fold modal matching,'' \emph{Proceedings of the VLDB Endowment}, vol.~18, no.~2, pp. 226--238, 2024.

\bibitem[Ritter-Gutierrez et~al.(2024)Ritter-Gutierrez, Huang, Wong, Ng, Lee, Chen, and Chng]{ritter2024dataset}
F.~Ritter-Gutierrez, K.-P. Huang, J.~H. Wong, D.~Ng, H.-y. Lee, N.~F. Chen, and E.~S. Chng, ``Dataset-distillation generative model for speech emotion recognition,'' \emph{arXiv preprint arXiv:2406.02963}, 2024.

\bibitem[Wu et~al.(2024)Wu, Zhang, Deng, and Russakovsky]{wu2024visionlanguage}
\BIBentryALTinterwordspacing
X.~Wu, B.~Zhang, Z.~Deng, and O.~Russakovsky, ``Vision-language dataset distillation,'' \emph{Transactions on Machine Learning Research}, 2024. [Online]. Available: \url{https://openreview.net/forum?id=6L6cD65pot}
\BIBentrySTDinterwordspacing

\bibitem[Sutton and Barto(2018)]{sutton2018reinforcement}
R.~S. Sutton and A.~G. Barto, \emph{Reinforcement learning: An introduction}.\hskip 1em plus 0.5em minus 0.4em\relax MIT press, 2018.

\bibitem[Pomerleau(1991)]{pomerleau1991efficient}
D.~A. Pomerleau, ``Efficient training of artificial neural networks for autonomous navigation,'' \emph{Neural computation}, vol.~3, no.~1, pp. 88--97, 1991.

\bibitem[Werbos(1990)]{werbos1990backpropagation}
P.~J. Werbos, ``Backpropagation through time: what it does and how to do it,'' \emph{Proceedings of the IEEE}, vol.~78, no.~10, pp. 1550--1560, 1990.

\bibitem[Monier et~al.(2020)Monier, Kmec, Laterre, Pierrot, Courgeau, Sigaud, and Beguir]{monier2020offline}
L.~Monier, J.~Kmec, A.~Laterre, T.~Pierrot, V.~Courgeau, O.~Sigaud, and K.~Beguir, ``Offline reinforcement learning hands-on,'' \emph{arXiv preprint arXiv:2011.14379}, 2020.

\bibitem[Schweighofer et~al.(2021)Schweighofer, Hofmarcher, Dinu, Renz, Bitto-Nemling, Patil, and Hochreiter]{schweighofer2021understanding}
\BIBentryALTinterwordspacing
K.~Schweighofer, M.~Hofmarcher, M.-C. Dinu, P.~Renz, A.~Bitto-Nemling, V.~P. Patil, and S.~Hochreiter, ``Understanding the effects of dataset characteristics on offline reinforcement learning,'' in \emph{Deep RL Workshop NeurIPS 2021}, 2021. [Online]. Available: \url{https://openreview.net/forum?id=A4EWtf-TO3Y}
\BIBentrySTDinterwordspacing

\bibitem[Van~der Maaten and Hinton(2008)]{van2008visualizing}
L.~Van~der Maaten and G.~Hinton, ``Visualizing data using t-sne.'' \emph{Journal of machine learning research}, vol.~9, no.~11, 2008.

\bibitem[Ross and Bagnell(2010)]{ross2010efficient}
S.~Ross and D.~Bagnell, ``Efficient reductions for imitation learning,'' in \emph{Proceedings of the thirteenth international conference on artificial intelligence and statistics}.\hskip 1em plus 0.5em minus 0.4em\relax JMLR Workshop and Conference Proceedings, 2010, pp. 661--668.

\bibitem[Mohri et~al.(2018)Mohri, Rostamizadeh, and Talwalkar]{mohri2018foundations}
M.~Mohri, A.~Rostamizadeh, and A.~Talwalkar, \emph{Foundations of machine learning}.\hskip 1em plus 0.5em minus 0.4em\relax MIT press, 2018.

\bibitem[Papamakarios et~al.(2017)Papamakarios, Pavlakou, and Murray]{papamakarios2017masked}
\BIBentryALTinterwordspacing
G.~Papamakarios, T.~Pavlakou, and I.~Murray, ``Masked autoregressive flow for density estimation,'' in \emph{Advances in Neural Information Processing Systems}, I.~Guyon, U.~V. Luxburg, S.~Bengio, H.~Wallach, R.~Fergus, S.~Vishwanathan, and R.~Garnett, Eds., vol.~30.\hskip 1em plus 0.5em minus 0.4em\relax Curran Associates, Inc., 2017. [Online]. Available: \url{https://proceedings.neurips.cc/paper_files/paper/2017/file/6c1da886822c67822bcf3679d04369fa-Paper.pdf}
\BIBentrySTDinterwordspacing

\bibitem[Haarnoja et~al.(2018)Haarnoja, Zhou, Abbeel, and Levine]{haarnoja2018soft}
T.~Haarnoja, A.~Zhou, P.~Abbeel, and S.~Levine, ``Soft actor-critic: Off-policy maximum entropy deep reinforcement learning with a stochastic actor,'' in \emph{International conference on machine learning}.\hskip 1em plus 0.5em minus 0.4em\relax PMLR, 2018, pp. 1861--1870.

\bibitem[Tarasov et~al.(2022)Tarasov, Nikulin, Akimov, Kurenkov, and Kolesnikov]{tarasov2022corl}
\BIBentryALTinterwordspacing
D.~Tarasov, A.~Nikulin, D.~Akimov, V.~Kurenkov, and S.~Kolesnikov, ``{CORL}: Research-oriented deep offline reinforcement learning library,'' in \emph{3rd Offline RL Workshop: Offline RL as a ''Launchpad''}, 2022. [Online]. Available: \url{https://openreview.net/forum?id=SyAS49bBcv}
\BIBentrySTDinterwordspacing

\end{thebibliography}
